\documentclass[twoside]{article}
\pdfoutput=1

\newif\iflong
\longfalse
\usepackage[accepted]{aistats2025}

\usepackage[round]{natbib}

\usepackage[round]{natbib}
\usepackage{hyperref}
\usepackage[mathscr]{euscript}
\usepackage{microtype}
\usepackage{graphicx}
\usepackage{booktabs}
\usepackage{xcolor}
\usepackage{amsmath}
\usepackage{amssymb}
\usepackage{mathtools}
\usepackage{amsthm,thmtools}
\usepackage{physics}
\usepackage{mycommands}
\usepackage{pascal-def}
\usepackage{thm-restate}
\usepackage[capitalise]{cleveref}
\usepackage[lined,boxed, ruled, vlined]{algorithm2e}
\usepackage{subcaption}
\usepackage{wrapfig}
\usepackage{multirow}
\usepackage{multicol}
\usepackage[subtle]{savetrees}

\newtheorem{theorem}{Theorem}
\newtheorem{lemma}[theorem]{Lemma}
\newtheorem{corollary}[theorem]{Corollary}
\newtheorem{proposition}[theorem]{Proposition}
\newtheorem*{example}{Example}

\newcommand{\intrinsic}{inherent}
\newcommand{\extrinsic}{dedicated}

\begin{document}

\twocolumn[
\runningtitle{Sample Compression Unleashed}
\aistatstitle{Sample Compression Unleashed:\\ New Generalization Bounds for Real Valued Losses}

\aistatsauthor{ Mathieu Bazinet \And Valentina Zantedeschi \And  Pascal Germain }

\aistatsaddress{ Université Laval \And  ServiceNow Research, Université Laval \And Université Laval } ]

\begin{abstract}
The sample compression theory provides generalization guarantees for predictors that can be fully defined using a subset of the training dataset and a (short) message string, generally defined as a binary sequence. Previous works provided generalization bounds for the zero-one loss, which is restrictive notably when applied to deep learning approaches. In this paper, we present a general framework for deriving new sample compression bounds that hold for real-valued unbounded losses. Using the Pick-To-Learn (P2L) meta-algorithm, which transforms the training method of any machine-learning predictor to yield sample-compressed predictors, we empirically demonstrate the tightness of the bounds and their versatility by evaluating them on random forests and multiple types of neural networks. 
\end{abstract}

\section{INTRODUCTION}
Sample compression theory, introduced by \cite{littlestone1986relating}, is based on the fundamental idea that ``compressing implies learning'' \citep{david_supervised_2016}. If it is possible to provably show that a learned model can be completely defined by a subset of the training dataset, then sample compression theory gives us generalization guarantees. The most well-known learning algorithms that comply with the sample compression framework are the support vector machine (SVM) \citep{boser_svm} and the perceptron \citep{rosenblatt1958perceptron, moran2020perceptron}; the relevant training subset being formed by the support vectors in the former case, and the points causing an update of the predictor in the latter case. More recently, \cite{snyder_sample_2020} and \cite{paccagnan_pick_learn_2023} have introduced the first sample compression results for neural networks.

The sample compression theory is rich and multiple different approaches exist. For example, \cite{hanneke_agnostic_2018,attias2023optimal,attias2024sample, ben2024private, david_supervised_2016,floyd1995sample,  hanneke_stable_2021, hanneke_efficient_2018, hanneke_sample_2019, hanneke2024list,  moran_sample_2015, rubinstein2012geometric} propose theoretical results relating the VC dimension \citep{vcdim} and the compression analysis. By relating the probability of \emph{change of compression} to the true risk, \cite{campi2023compression,paccagnan_pick_learn_2023} express very tight guarantees for the consistent case, i.e., when the error on the training set is zero. Finally, \cite{marchand2002set, marchand2003set, graepel_pac-bayesian_2005,shah_margin-sparsity_2005,Laviolette2009LearningTS, marchand2005learning,hussain2007, shah_sample_2007} give computable risk certificates valid even in the non-consistent case.

In this paper, we build on the setting of \cite{shah_margin-sparsity_2005}.
Their sample-compression bound is based on the binomial test-set bound of \cite{langford2005tutorial}, which by definition is the tightest test-set bound for the zero-one loss under the sole \emph{i.i.d.}\ assumption. However, the use of the zero-one loss restricts its application to supervised classification problems. By leveraging proof techniques from the PAC-Bayesian literature, we extend the framework to real-valued losses and open the way to obtaining bounds directly for the cross-entropy loss \citep{perez2021tighter} and unbounded losses  \citep{haddouche2021pac, casado2024pac, rodriguez2024more}, for example under the sub-Gaussian assumption \citep{Kahane1960}. Finally, we train deep neural networks and random forests with Pick-To-Learn (P2L) \citep{paccagnan_pick_learn_2023}, a meta-algorithm that modifies the training loop of a model to yield a sample-compressed predictor, and assess the tightness of our bounds in different settings.

Of note, a major asset of our sample-compress bounds is that they do not depend on the number of learnable parameters. Two models of different sizes can achieve the same guarantees as long as they achieve the same empirical loss using the same amount of data. This lets us train large models such as DistilBERT \citep{sanh2019distilbert} and still achieve tight generalization bounds.

The paper is organized as follows. In \cref{sec:background}, we present the sample compression theory and the meta-algorithm Pick-To-Learn (P2L) \citep{paccagnan_pick_learn_2023}. In \cref{sec:main}, we first present \cref{thm:main_results}, a new general sample-compression theorem that holds for any real-valued losses. Leveraging the comparator functions of PAC-Bayes theory \citep{mcallester1998some}, we present two new sample compression bounds for losses in the interval unit, \cref{corr:catoni} and \cref{corr:kl_bound}, which respectively yield the tightest bound in theory and in practice. Then, we present \cref{corr:linear}, which holds for any unbounded losses, under the assumption that the moment-generating function is bounded. We finish this section by proving the tightness of our results over the previous state-of-the-art sample compression bound. Finally, in \cref{sec:experiments}, we empirically show the tightness of our bounds by training deep neural networks on image and text classification problems with P2L.
We adapt P2L to regression, train regression trees and forests with this modified algorithm and provide the first sample compression generalization bounds for tree-based regression predictors. 

\section{BACKGROUND AND NOTATION}\label{sec:background}

We are interested in the supervised learning framework. Let $(\bx_1, y_1), \ldots, (\bx_n, y_n)$ be a sequence of $n$ datapoints sampled \emph{i.i.d.}\
(independently and identically distributed) from an unknown distribution $\calD$ over $\R^d \times \calY$. The dataset $S = \{(\bx_i, y_i)\}_{i=1}^n$ is generated with the sequence of datapoints.\footnote{In this paper, we do not consider repeated datapoints. However, all definitions and results could be easily adapted to use multisets to account for repetitions, similarly to the work of \citet{campi2023compression}.} The targets are defined by the task at hand, with $\calY \in \{-1,+1\}$ for binary classification tasks and $\calY \subseteq \R$ for regression tasks.
 In this section, we focus on binary classification problems, but in \cref{sec:main}, we study both classification and regression settings.

Let $\calH$ be a family of predictors $h : \calX \to \calY$. Let $A : \bigcup_{k=1}^{\infty} (\calX \times \calY)^k \to \calH$ be a learning algorithm that takes a dataset $S$ and returns a predictor $A(S)$. We consider the zero-one loss function $\ell^{0\textrm-1}(h, \bx, y) = \indicator[h(\bx) \neq y]$, with $\indicator[a]=1$ if the predicate $a$ is true and $0$ otherwise. Then, the true risk of the hypothesis $h$ is defined as 
$$R_{\calD}(h) \ = \Prob_{(\bx, y) \sim \calD}(h(\bx) \neq y) \ = \E_{(\bx, y) \sim \calD} \indicator[h(\bx) \neq y]$$
and, for a realization $S \sim \calD^n$, its empirical risk is defined as $\widehat{R}_{S}(h) = \tfrac{1}{n}\sum_{i=1}^n \indicator[h(\bx_i) \neq y_i]$. 

Since the distribution $\calD$ is unknown, the true risk of a hypothesis cannot be computed. However, it can be upper bounded with high probability, using generalization bounds derived from statistical learning theories such as the sample compression theory. 
\subsection{Sample compression theory}\label{sec:marchand}
Let the predictor $h=A(S)$ be the output of a learning algorithm $A$ applied to a dataset $S$. In order to obtain guarantees on the generalization performance of $h$ using the sample compression theory, we need to be able to uniquely define $h$ as a function (the reconstruction function) of a subset of~$S$ (the compression set) and a complementary sequence of information (the message).

The compression set $S_{\bfi}$ is defined using a vector of indices $\bfi = \qty(i_1, i_2, \ldots, i_{\m})$, where the indices are ordered such that $1 \leq i_1 < i_2 < \ldots < i_{\m} \leq n$. The vector $\bfi$ belongs in the set of all possible vectors composed of the natural numbers 1 through $n$, denoted 
\begin{equation*}
    \scriptP(n) = \bigg\{\emptyset,\{1\}, \{2\}, \ldots, \{n\}, \{1,2\}, \ldots, \{1,n\},\ldots, \{1,2,\ldots, n\}\bigg\}.
\end{equation*}
Using this notation, $\bfi$ indicates the datapoints of $S$ that are present in $S_{\bfi}$ : 
\begin{equation*}
    S_{\bfi} = \left\{(\bx_{i_1},y_{i_1}), \ldots, (\bx_{i_{\m}}, y_{i_{\m}})\right\} \subseteq S\,.
\end{equation*}
Moreover, we define the complement vector $\bfi^c \in \scriptP(n)$ such that $S_{\bfi^c} = S \setminus S_{\bfi}$ and $|\mathbf{i}^c| = n-\m$.

The message $\sigma$ is chosen in a set $M(\bfi)$, which contains all relevant messages associated to the compression set $\bfi$. The message is a complementary source of information needed to reconstruct the predictor.

A predictor $h$ is called a sample-compressed predictor if there exists a vector $\bfi \in \scriptP(n)$ and (optionally) a message $\sigma \in M(\bfi)$ such that $h = \scriptR(S_{\bfi}, \sigma)$, where $\scriptR : \bigcup_{m \leq n}(\calX \times \calY)^m \times \bigcup_{\bfi \in \scriptP(n)} M(\bfi) \to \overline{\calH}$ is a data-independent deterministic reconstruction function and $\overline{\calH} \subseteq \calH$ is a discrete set of sample-compressed predictors.

In this paper, we distinguish two categories of reconstruction functions: \intrinsic{} and \extrinsic{}. An \intrinsic{} reconstruction function is used when a learning algorithm $A$ is its own reconstruction function. 
An algorithm $A$ is an inherent reconstruction function when, given a dataset $S$ and its compression set $S_{\bfi}$, the following equality holds
$A(S) = A(S_{\bfi})$. The most well-known example of an \intrinsic{} reconstruction function is the SVM \citep{boser_svm}. Other examples of \intrinsic{} reconstruction functions are the perceptron \citep{rosenblatt1958perceptron} and Pick-To-Learn \citep{paccagnan_pick_learn_2023}. On the other hand, \extrinsic{} reconstruction functions are used when an algorithm cannot be used to reconstruct the learned predictor from a compression set. The \extrinsic{} reconstruction function is different from the learning algorithm $A$ and is generally hand-crafted to suit $A$. The reconstruction function of the SCM \citep{marchand2002set} and all its iterations \citep[e.g.,][]{marchand2003set, marchand2005learning, shah_margin-sparsity_2005,kestler2006learning, hussain2007, Drouin_2019} are examples of \extrinsic{} reconstruction functions. 

We provide an example of a dedicated reconstruction function for a very simple predictor, the decision stump.
\begin{example}[\cite{shah2011feature}]
Given a datapoint $\bx' = (x'_1, \ldots, x'_d)$, a direction $\diamond \in \{-1, +1\}$ and an index $1\leq k \leq d$, the stump is defined $f_{(\bx', d, k)}(\bx) = \indicator\qty[\diamond \cdot (x_k - x'_k) > 0]$. To learn a decision stump over a dataset $S$, each combination of $\bx' \in S$, $\diamond \in \{-1, +1\}$ and $0\leq k \leq d$ is tested. Once the decision stump is learned, it is completely defined by the datapoint $\bx'$, the direction $\diamond$ and the index $k$. Our compression set is $S_{\bfi} = \{\bx'\}$ and the message is $\sigma = \{\diamond, k\}$.
With the compression set and the message, we can fully reconstruct the stump with $\scriptR(S_{\bfi}, \{\diamond, k\}) = f_{(S_{\bfi}, \diamond, k)}$.
\end{example}
Let $P_{\overline{\calH}}$ be a distribution over $\overline{\calH}$, such that $\sum_{h \in \overline{\calH}} P_{\overline{\calH}}(h) \leq 1$. As all sample-compressed predictors are uniquely defined using the index vector and the message, we choose the distribution $P_{\overline{\calH}}$ to be a product of two distributions $P_{\overline{\calH}}(\scriptR(S_{\bfi}, \sigma)) = P_{\scriptP(n)}(\bfi) P_{M(\bfi)}(\sigma)$, with $P_{\scriptP(n)}$ a distribution on $\scriptP(n)$ and $P_{M(\bfi)}$ a distribution on $M(\bfi)$. Following previous works \citep[e.g.][]{marchand2005learning}, we require the  distribution $P_{\calH}$ to be data-independent, in order to avoid further assumptions. Without any information on the data, we generally set $P_{M(\bfi)}$ to a uniform distribution. As for the distribution $P_{\scriptP(n)}$, it is usually set to penalize larger compression sets~\citep{shah_margin-sparsity_2005,marchand2002set,marchand2005learning}. For any size of compression set $\m$, there are $\smqty(n \\ \m)$ different possible compression sets. We set the distribution $P_{\scriptP(n)}(\bfi)$ to be $\smqty(n \\ \m)^{-1} \zeta(\m)$, with $\zeta(m) = \tfrac{6}{\pi^2}(m+1)^{-2}$. This choice is discussed by \cite{marchand2005learning}.

We now present the sample compression bound of \cite{shah_margin-sparsity_2005}. This result is derived using the binomial test-set bound of \cite{langford2005tutorial}, which by definition is the tightest test-set bound for the zero-one loss under the sole i.i.d.\ assumption.
\begin{theorem}[\citet{shah_margin-sparsity_2005}, Theorem 1]\label{thm:binom_tail}
    For any distribution $\calD$ over $\calX \times \calY$, for any family of set of messages $\{M(\bfi)\, | \bfi \in \scriptP(n)\}$, for any deterministic reconstruction function $\scriptR$ that outputs sample-compressed predictors $h \in \overline{\calH}$ and for any $\delta \in (0,1]$, with probability at least $1-\delta$ over the draw of $S \sim \calD^n$, we have
    \begin{align*}
        &\forall \bfi \in \scriptP(n), \sigma \in M(\bfi): \\
        &R_{\calD}(\scriptR(S_{\bfi}, \sigma)) \leq \overline{\emph{Bin}}\qty(\kappa,|\bfi^c|, \mqty(n \\ \m)^{-1} \zeta(\m) P_{M(\bfi)}(\sigma)\delta),
    \end{align*}
    with $\kappa = |\bfi^c|\widehat{R}_{S_{\bfi^c}}(\scriptR(S_{\bfi}, \sigma))$ and 
    \begin{equation*}
        \overline{\emph{Bin}}\qty(k,m, \delta) = \sup_{r \in [0,1]} \left\{\sum_{i=0}^k \mqty(m \\ i) r^i (1-r)^{m-i} \geq \delta \right\}.
    \end{equation*}
\end{theorem}

This theorem can be applied to any family of sample-compressed predictors, such as the support vector machine \citep{boser_svm}, the perceptron \citep{rosenblatt1958perceptron} and the set covering machine \citep{marchand2002set}. To apply this theorem to neural networks, one must design a reconstruction function outputting neural networks. To this end, \cite{snyder_sample_2020} propose to reparameterize a 2-layer LeakyReLU network in order to obtain ``support vectors'', which become the compression set of the reconstructed network. The following section presents a more general approach proposed by \cite{paccagnan_pick_learn_2023}.

\subsection{Pick-To-Learn}\label{sec:p2l}

Conceptualized by \cite{paccagnan_pick_learn_2023}, Pick-To-Learn (P2L) is a model-agnostic meta-algorithm that trains any model in such a way that it becomes a sample-compressed predictor. This algorithm is specifically designed for the generalization bound of \cite{campi2023compression}, which holds only for sample compressed predictors in the \emph{consistent case}, i.e., when $\widehat{R}_{S_{\bfi^c}}(\scriptR(S_{\bfi}, \sigma)) {=} 0$. 

To obtain sample-compressed predictors, P2L iteratively builds the compression set and trains the model on it. Starting with an initial predictor $h_0$, P2L tests the model on the whole dataset, picks the datapoint over which the model got the largest loss value, and adds it to the compression set. Then, using a learning algorithm~$A$, P2L trains the model on the newly created compression set. The previous steps are repeated until the model achieves zero errors on the training set $S_{\bfi^c}$ (excluding the compression set datapoints), which is equivalent to stopping when the cross-entropy loss ($\ell^{\textrm{x-e}}$) becomes smaller than $-\ln(0.5)$. We present P2L in Algorithm~\ref{alg:p2l}.
\SetKwInOut{Init}{Initialize}
\begin{algorithm}[!h]
\SetAlgoLined
\DontPrintSemicolon
\Init{$S_{\bfi} \leftarrow \emptyset $}
\Init{$h_{\bfi} \leftarrow h_0$}
\Init{$(\overline{\bx}, \overline{y}) \leftarrow \argmax_{(\bx,y) \in S} \ell^{\textrm{x-e}}(h_0, \bx, y)$}
\While{$-\ln(0.5) \leq \ell^{\emph{x-e}}(h_{\bfi}, \overline{\bx}, \overline{y})$}{
$S_{\bfi} \leftarrow S_{\bfi} \cup \{(\overline{\bx}, \overline{y})\}$\;
$h_{\bfi} \leftarrow A(S_{\bfi})$\;
$(\overline{\bx}, \overline{y}) \leftarrow \argmax_{(\bx,y) \in S_{\bfi^c}} \ell^{\textrm{x-e}}(h_{\bfi}, \bx, y)$
}
\Return $h_{\bfi}$
\caption{Pick-To-Learn (P2L)}\label{alg:p2l}
\end{algorithm}

Leveraging from the theoretical results of \cite{campi2023compression}, \cite{paccagnan_pick_learn_2023} derived a theorem specifically for the P2L algorithm.
\begin{theorem}[\citet{paccagnan_pick_learn_2023}, Theorem 4.2] \label{thm:p2l}
    Let $h_{\bfi} = \scriptR(S_{\bfi}, \emptyset)$ be the output of P2L. For any $\delta \in (0,1)$, with probability at least $1-\delta$ over the draw of $S \sim \calD^n$, we have
    \begin{equation*}
        R_{\calD}(h_{\bfi}) \ \leq \ \overline{\varepsilon}(\m, \delta)\,,
    \end{equation*}
where, for $k=0,1,\ldots, n-1$, $\overline{\varepsilon} (k,\delta)$ is the unique solution to the equation $\Psi_{k,\delta}(\varepsilon) = 1$ in the interval $[\frac{k}{n}, 1]$, %
with
\begin{align*}
    \Psi_{k,\delta}(\varepsilon) &= \frac{\delta}{2n} \hspace{1.5mm} \sum_{m=k}^{n-1} \ \ \hspace{-0.2mm} \frac{\smqty(m \\ k)}{\smqty(n \\ k)}(1-\varepsilon)^{-(n-m)}\\& + \frac{\delta}{6n} \sum_{m=n+1}^{4n} \frac{\smqty(m \\ k)}{\smqty(n \\ k)}(1-\varepsilon)^{-(n-m)}\,,
 \end{align*}
 and $\overline{\varepsilon}(n, \delta) = 1$.
\end{theorem}

Note that the value of the previous bound is completely determined by $|\bfi|$, the size of the compression set. The faster P2L obtains zero errors (in terms of the number of iterations performed by Algorithm~\ref{alg:p2l}), the better the bound.

\section{A GENERAL SAMPLE-COMPRESSION BOUND}\label{sec:main}
Let $\calH$ be a family of predictors $h : \calX \to \overline{\calY}$, where $\overline{\calY} \,{\supseteq} \calY$ is a convex hull of $\calY$. For example, $[-1,1]$ is the convex hull of $\{-1,+1\}$. We consider a loss function $\ell : \calH \times \calX \times \calY \to \R$. Then, the true risk of the hypothesis $h$ is defined as $\calL_{\calD}(h) = \E_{(\bx, y) \sim \calD} \ell(h, \bx, y)$ and, for a realization $S \sim \calD^n$, its empirical risk is defined as $\hatL_{S}(h) = \tfrac{1}{n}\sum_{i=1}^n \ell(h, \bx_i, y_i)$. This setting is a generalization of the setting of \cref{sec:background}. As \cref{thm:binom_tail} only holds for the zero-one loss, we need new results to extend the sample-compression theory to this setting.

To extend the work of \citet{shah_margin-sparsity_2005} to real-valued losses, we introduce a \emph{comparator function} $\Delta : \R \times \R \to \R$ and provide a new result inspired by the general PAC-Bayes bound \citep{germain2009pac}. \cref{thm:main_results} presents a new general sample-compress bound that holds for any real-valued losses, extending the applicability of the sample-compression theory. The theorem is followed by a proof sketch highlighting the main steps, and the full proof is given in \cref{app:proofs}. 
\begin{restatable}{theorem}{mainresult}\label{thm:main_results}
For any distribution $\calD$ over $\calX \times \calY$, for any family of set of messages $\{M(\bfi)\, | \bfi \in \scriptP(n)\}$, for any deterministic reconstruction function $\scriptR$ that outputs sample-compressed predictors $h \in \overline{\calH}$, for any loss $\ell: \calH \times \calX \times \calY \to \R$, for any comparator function $\Delta : \R\times \R \to \R$ and for any $\delta \in (0,1]$, with probability at least $1-\delta$ over the draw of $S \sim \calD^n$, we have 
\iflong
 \begin{align*}
&\forall \mathbf{i} \in \scriptP(n), \sigma \in M(\bfi): \Delta\qty(\hatL_{S_{\bfi^c}}(\scriptR(S_{\bfi},\sigma)), \mathcal{L}_{\calD}(\scriptR(S_{\bfi},\sigma))) \leq \frac{1}{|\bfi^c|}\qty[\log \mqty(n \\ \m) + \log\qty(\frac{\mathcal{E}_{\Delta}\qty(\mathbf{i},\sigma)}{\zeta(\m)P_{M(\bfi)}(\sigma)\delta})]\,,
\end{align*}
\else
 \begin{align*}
&\forall \mathbf{i} \in \scriptP(n), \sigma \in M(\bfi): \\
&\Delta\qty(\hatL_{S_{\bfi^c}}(\scriptR(S_{\bfi},\sigma)), \mathcal{L}_{\calD}(\scriptR(S_{\bfi},\sigma)))\\
&~~~~\leq \frac{1}{|\bfi^c|}\qty[\log \mqty(n \\ \m) + \log\qty(\frac{\mathcal{E}_{\Delta}\qty(\mathbf{i},\sigma)}{\zeta(\m)P_{M(\bfi)}(\sigma)\delta})]\,,
\end{align*}\fi
with 
\begin{equation*}
    \Ecal_{\Delta}(\bfi, \sigma) = \E_{T_{\bfi} \sim \calD^{\m}} \E_{T_{\bfi^c} \sim \calD^{|\mathbf{i}^c|}} e^{|\mathbf{i}^c|\Delta\qty(\hatL_{T_{\bfi^c}}(\scriptR(T_{\bfi}, \sigma)), \calL_{\calD}(\scriptR(T_{\bfi}, \sigma)))}.
\end{equation*}
\end{restatable}
\begin{proof}[Proof Sketch]
    For all $\mathbf{i} \in \scriptP(n)$, $ \sigma \in M(\bfi)$, $\epsilon > 0$, using Chernoff's bound with $t>0$, we have
    \begin{align*}
        &\Prob_{S \sim \calD^n}\qty(\Delta\qty(\hatL_{S_{\bfi^c}}(\scriptR(S_{\bfi},\sigma)), \mathcal{L}_{\calD}(\scriptR(S_{\bfi},\sigma))) > \epsilon)  \numberthis \label{eq:chernoff_epsilon}\\
        &\leq e^{-t\epsilon} \E_{S \sim \calD^n} e^{t\Delta\qty(\hatL_{S_{\bfi^c}}(\scriptR(S_{\bfi},\sigma)), \mathcal{L}_{\calD}(\scriptR(S_{\bfi},\sigma)))}\\
        &= e^{-t\epsilon} \E_{S_{\bfi} \sim \calD^{\m}} \E_{S_{\bfi^c} \sim \calD^{|\mathbf{i}^c|}}e^{t\Delta\qty(\hatL_{S_{\bfi^c}}(\scriptR(S_{\bfi},\sigma)), \mathcal{L}_{\calD}(\scriptR(S_{\bfi},\sigma)))}
    \end{align*}
where the last equality requires \emph{i.i.d.} datapoints. For any $\delta_{\bfi}^{\sigma} \in (0,1]$, we define
\begin{equation}
    \delta_{\bfi}^{\sigma} = e^{-t\epsilon} \hspace{-2mm}\E_{S_{\bfi} \sim \calD^{\m}} \E_{S_{\bfi^c} \sim \calD^{|\mathbf{i}^c|}} \hspace{-2mm}e^{t\Delta\qty(\hatL_{S_{\bfi^c}}(\scriptR(S_{\bfi},\sigma)), \mathcal{L}_{\calD}(\scriptR(S_{\bfi},\sigma)))}\label{eq:delta_i_cgf}
\end{equation}
and solve for $\epsilon$, using $t=|\mathbf{i}^c|$. 
The obtained solution is used to replace the $\epsilon$ in \cref{eq:chernoff_epsilon}, which gives a bound valid with probability~$\delta_{\bfi}^{\sigma}$ for every single predictor $\scriptR(S_{\bfi},\sigma)$.
By setting $\delta_{\bfi}^{\sigma} = P_{\scriptP(n)}(\bfi)P_{M(\bfi)}(\sigma)\delta$ and applying a union bound over all $\mathbf{i} \in \scriptP(n)$,  $\sigma \in M(\bfi)$, the final result holds uniformly with probability $\delta$ for all predictors outputted by $\scriptR$.
\end{proof}
\Cref{thm:main_results} holds for any comparator function $\Delta$ such that $\mathcal{E}_{\Delta}$ is finite for any pair $(\bfi, \sigma)$. Although bounding $\Ecal_{\Delta}$ can be challenging, it was extensively studied for convex functions in PAC-Bayesian theory \citep[e.g.,][]{mcallester1998some, maurer2004note, casado2024pac, hellstrom2024comparing}. 
We leverage this theory and present novel corollaries for the three most well-known comparators.

First of all, we present a bound using the comparator $\Delta_C(q,p) = -\ln\qty(1-p(1-e^{-C})) - Cq$. The family of bounds $\{\Delta_C : C > 0\}$ is commonly referred to as ``Catoni bounds'' \citep{catoni2007pac} in the PAC-Bayes literature. 

\begin{restatable}{corollary}{catoni}\label{corr:catoni}
    In the setting of \cref{thm:main_results}, for any $C>0$, for any loss function $\ell : \calH \times \calX \times \calY \to [0,1]$, with probability at least $1-\delta$ over the draw of $S \sim \calD^n$, we have
\begin{align*}
&\forall \bfi \in \scriptP(n),\sigma \in M(\bfi):\,\mathcal{L}_{\calD}(\scriptR(S_{\bfi},\sigma)) \leq \frac{1-\exp(-\epsilon_C(\bfi, \sigma, \delta))}{1-e^{-C}},
\end{align*}
with 
\iflong
\begin{align*}
    &\epsilon_C(\bfi, \sigma, \delta) = C\hatL_{S_{\bfi^c}}(\scriptR(S_{\bfi},\sigma)) + \frac{1}{n-\m}\qty[\log\mqty(n \\ \m) +  \log\qty(\frac{1}{\zeta(\m)P_{M(\bfi)}(\sigma)\delta})].
\end{align*}
\else
\begin{align*}
    &\epsilon_C(\bfi, \sigma, \delta) = C\hatL_{S_{\bfi^c}}(\scriptR(S_{\bfi},\sigma)) \\
    &~~~~~~+ \frac{1}{n-\m}\qty[\log\mqty(n \\ \m) +  \log\qty(\frac{1}{\zeta(\m)P_{M(\bfi)}(\sigma)\delta})].
\end{align*}\fi
\end{restatable}

For $0\leq q,p \leq 1$, there exists $C^* = \argsup_{C>0} \Delta_C(q,p)$ such that $\Delta_{C^*}$ gives the tightest PAC-Bayesian bounds \citep{foong2021tight}. This result also holds true for \cref{thm:main_results}, when restricted to proper, convex and lower semicontinuous comparator functions $\Delta : [0,1]\times [0,1] \to \R$. Unfortunately, the $\Delta_C$ bound hold for only one value of $C$, chosen prior to seeing $S$. With a union bound argument, we can consider multiple parameters $C$ simultaneously, but there is no guarantee that $C^*$ is in this set. To circumvent this problem, we can use the binary Kullback-Leibler divergence comparator function 
    $\kl(q,p) = q\ln\frac{q}{p} + (1-q)\ln\frac{1-q}{1-p},$
which is equivalent to $\Delta_{C^*}(q,p)$, as per the following proposition.
\begin{proposition}[\citet{germain2009pac}, Proposition~2.1]\label{prop:kl_catoni_germain}
    For any $0 \leq q \leq p < 1$, we have 
    $%
        \sup_{C \geq 0} \Delta_C(q,p) = \kl(q,p).
    $%
\end{proposition}

In practice, even with the term $1 = \Ecal_{\Delta_C}(\bfi, \sigma) \leq \Ecal_{\kl}(\bfi, \sigma) = 2\sqrt{n-\m}$, the $\kl$ bound stated below (\cref{corr:kl_bound}) usually yield tighter bounds than the $\Delta_C$ bound (\cref{corr:catoni}), as the optimal value $C^*$ is unlikely to be selected before computing the bound. Moreover, the $\kl$ is known to be optimal for $[0,1]$-valued losses, as per the results of \cite{hellstrom2024comparing}.

\begin{restatable}{corollary}{klbound}\label{corr:kl_bound}
In the setting of \cref{thm:main_results}, for any loss function $\ell : \calH \times \calX \times \calY \to [0,1]$, with probability at least $1-\delta$ over the draw of $S \sim \calD^n$, we have
\iflong
\begin{align*}
&\forall \mathbf{i} \in \scriptP(n), \sigma \in M(\bfi): \mathcal{L}_{\calD}(\scriptR(S_{\bfi},\sigma)) \leq \kl^{-1}\qty(\hatL_{S_{\bfi^c}}(\scriptR(S_{\bfi},\sigma)), \epsilon_{\kl}(\bfi, \sigma, \delta)),
\end{align*}
\else
\begin{align*}
&\forall \mathbf{i} \in \scriptP(n), \sigma \in M(\bfi): \\
& \mathcal{L}_{\calD}(\scriptR(S_{\bfi},\sigma)) \leq \kl^{-1}\qty(\hatL_{S_{\bfi^c}}(\scriptR(S_{\bfi},\sigma)), \epsilon_{\kl}(\bfi, \sigma, \delta)),
\end{align*}\fi
with 
\ $%
    \kl^{-1}\qty(q, \epsilon) = \argsup_{0\leq p \leq 1} \left\{\kl(q,p) \leq \epsilon \right\}
\,$ \ %
and 
\begin{equation*}
    \epsilon_{\kl}(\bfi, \sigma, \delta) = \frac{1}{n-\m} \left[\log \mqty(n \\ \m) + \log\qty(\frac{2\sqrt{n-\m}}{\zeta(\m)P_{M(\bfi)}(\sigma)\delta})\right].
\end{equation*}
\end{restatable}
Both \cref{corr:catoni} and \cref{corr:kl_bound} hold for losses bounded in $[0,1]$. Using the linear function $\Delta_{\lambda}(q,p) = \lambda (p-q)$, we can extend this sample compression framework to unbounded losses provided that $\Ecal_{\Delta_{\lambda}}$ is bounded. As an example, we present a result for sub-Gaussian losses \citep{Kahane1960}.

\begin{restatable}{corollary}{linearloss}\label{corr:linear}
In the setting of \cref{thm:main_results}, for any $\lambda>0$, with a $\varsigma^2$-sub-Gaussian loss function $\ell : \calH \times \calX \times \calY \to \R$, with probability at least $1-\delta$ over the draw of $S \sim \calD^n$, we have
\iflong
\begin{align*}
&\forall \mathbf{i} \in \scriptP(n), \sigma \in M(\bfi): \mathcal{L}_{\calD}(\scriptR(S_{\bfi},\sigma)) \leq \hatL_{S_{\bfi^c}}(\scriptR(S_{\bfi},\sigma)) + \frac{\lambda \varsigma^2}{2}+ \frac{1}{\lambda(n-\m)}\qty[\log \mqty(n \\ \m) + \log\qty(\frac{1}{\zeta(\m)P_{M(\bfi)}(\sigma)\delta})].
\end{align*}
\else
\begin{align*}
&\forall \mathbf{i} \in \scriptP(n), \sigma \in M(\bfi): \\
&\mathcal{L}_{\calD}(\scriptR(S_{\bfi},\sigma)) \leq \hatL_{S_{\bfi^c}}(\scriptR(S_{\bfi},\sigma)) + \frac{\lambda \varsigma^2}{2}\\
&~~+ \frac{1}{\lambda(n-\m)}\qty[\log \mqty(n \\ \m) + \log\qty(\frac{1}{\zeta(\m)P_{M(\bfi)}(\sigma)\delta})].
\end{align*}\fi
\end{restatable}
Note that this result encompasses bounded losses with a range of $[a,b]$, as they are sub-Gaussian with~$\varsigma = \frac{b-a}{2}$.
It can be extended to the hypothesis-dependent range condition of \cite{haddouche2021pac}, any unbounded losses under model-dependent assumptions \citep{casado2024pac} or more general tail behaviors \citep{rodriguez2024more}. 
\begin{table*}[!h]
\caption{ Results for the CNNs trained using P2L on the binary MNIST problems. The results displayed obtained the tightest P2L bound. All metrics presented are in percent (\%).} \label{tab:binary_mnist_cnn}
\begin{center}
\resizebox{2\columnwidth}{!}{%
\begin{tabular}{cccccccc} \toprule
Dataset & Validation error &	Test error&	$\kl$ bound	&Binomial bound	&P2L bound	&$\m/n$&	Baseline test error \\
\midrule
MNIST08 &	0.33$\pm$0.17 &	0.25$\pm$0.10	& 5.05$\pm$0.16 &	5.00$\pm$0.16	& 1.04$\pm$0.04 &	0.87$\pm$0.03 & 0.22$\pm$0.05\\
MNIST17&	0.20$\pm$0.08	& 0.38$\pm$0.16 &	4.33$\pm$0.21	& 4.29$\pm$0.21 &	0.86$\pm$0.05 &	0.72$\pm$0.04 &	0.17$\pm$0.08 \\ 
MNIST23&	0.39$\pm$0.12	& 0.27$\pm$0.10 &	8.20$\pm$0.34	& 8.15$\pm$0.34 &	1.86$\pm$0.09	& 1.61$\pm$0.09&	0.16$\pm$0.05 \\ 
MNIST49&	0.82$\pm$0.11	& 0.77$\pm$0.17 &	10.52$\pm$0.37 &	10.47$\pm$0.37 &	2.53$\pm$0.11 &2.23$\pm$0.10 &	0.44$\pm$0.08 \\ 
MNIST56&	0.46$\pm$0.12	& 0.47$\pm$0.15 &	6.29$\pm$0.22	& 6.24$\pm$0.22 &	1.35$\pm$0.06	& 1.15$\pm$0.05 &	0.30$\pm$0.05 \\
\bottomrule
\end{tabular}
}
\end{center}
\end{table*}
\begin{table*}[!h]
\caption{ Results for the CNNs trained using P2L on the binary MNIST problems and stopped at the iteration with the minimum $\kl$ bound. The results displayed obtained the tightest $\kl$ bound. %
Metrics are in percents (\%).
} \label{tab:binary_mnist_cnn_early_stopped}
\begin{center}
\resizebox{2\columnwidth}{!}{%
\begin{tabular}{cccccccc}
\toprule
Dataset & Validation error &	Test error&	$\kl$ bound	&Binomial bound	& Train error & $\m/n$&	Baseline test error \\
\midrule 
MNIST08  & 0.49$\pm$0.39 & 0.49$\pm$0.26 & 4.71$\pm$0.25 & 5.33$\pm$0.62 & 0.24$\pm$0.23 & 0.62$\pm$0.14 &0.22$\pm$0.05	\\
MNIST17& 0.45$\pm$0.18 & 0.48$\pm$0.11 & 3.70$\pm$0.21 & 4.37$\pm$0.11 & 0.23$\pm$0.08 & 0.43$\pm$0.08	& 0.17$\pm$0.08 \\ 
MNIST23 & 0.74$\pm$0.28 & 0.84$\pm$0.21 & 6.56$\pm$0.38 & 8.09$\pm$0.64 & 0.64$\pm$0.32& 0.77$\pm$0.20	& 0.16$\pm$0.05 \\ 
MNIST49&  1.16$\pm$0.31 & 1.13$\pm$0.24 & 8.60$\pm$0.46 & 9.61$\pm$0.68 & 0.51$\pm$0.28 & 1.26$\pm$0.23 & 0.44$\pm$0.08  \\ 
MNIST56& 0.94$\pm$0.09 & 0.70$\pm$0.20 & 5.42$\pm$0.31 & 6.49$\pm$0.81 & 0.43$\pm$0.23 & 0.65$\pm$0.10	&  0.30$\pm$0.05\\
\bottomrule
\end{tabular}
}
\end{center}
\vspace{-0.3cm}
\end{table*}
\subsection{Behavior in the consistent case}
In this section, we present new theoretical results that justify the tightness of the bounds observed in \cref{sec:experiments}, in which we train different types of models with P2L. By construction, P2L is designed to stop when the complement error $\widehat{R}_{S_{\bfi^c}}(\scriptR(S_{\bfi},\sigma))$ is zero. In this consistent setting, where the predictor always finishes training with zero errors, we can prove that Corollaries~\ref{corr:catoni} and~\ref{corr:kl_bound} are tight upper bounds of \cref{thm:binom_tail}. 

The first result, presented in \cref{thm:consistent_case_catoni}, states that the $\Delta_C$ bound is an arbitrarily tight upper bound of the binomial tail inversion bound of \citet{shah_margin-sparsity_2005}. Indeed, in the following theorem, we show that \cref{corr:catoni} is minimized by $C$ tending to $\infty$ and is equal to \cref{thm:binom_tail} in the limit of $C\to \infty$.

\begin{restatable}{theorem}{consistentcasecatoni}\label{thm:consistent_case_catoni}
In the consistent case, i.e.\ when $\widehat{R}_{S_{\bfi^c}}(\scriptR(S_{\bfi},\sigma)) = 0$, \cref{corr:catoni} is arbitrarily close to the binomial tail inversion of \cref{thm:binom_tail}. Indeed, we have
\begin{align}
    \overline{\rm{Bin}}\qty(0,|\bfi^c|, \delta_{\bfi}^{\sigma}) \ &= \inf_{C > 0}\left\{ \frac{1-\exp\qty(-\frac{1}{|\bfi^c|} \ln \frac{1}{\delta_{\bfi}^{\sigma}})}{1-e^{-C}} \right\} \label{eq:inf_bin_catoni}\\
    &=\lim_{C \to \infty}\left\{ \frac{1-\exp\qty(-\frac{1}{|\bfi^c|} \ln \frac{1}{\delta_{\bfi}^{\sigma}})}{1-e^{-C}} \right\} \label{eq:lim_bin_catoni}
\end{align}
with $\delta_{\bfi}^{\sigma} = \smqty(n \\ \m)^{-1} \zeta (\m)P_{M(\bfi)}(\sigma) \delta$. 
\end{restatable}

The following theorem states that the Kullback-Leibler divergence bound is a tight upper bound of the binomial tail inversion bound, up to a constant. 

\begin{restatable}{theorem}{consistentcasekl}\label{thm:consistent_case_kl}
In the consistent case, i.e.\ when $\widehat{R}_{S_{\bfi^c}}(\scriptR(S_{\bfi},\sigma)) = 0$, \cref{corr:kl_bound} is a tight upper bound of \cref{thm:binom_tail} up to a constant $K(m, \delta)$. Indeed, we have
\begin{align}
    \overline{\rm{Bin}}\qty(0,|\bfi^c|, \delta_{\bfi}^{\sigma})\  \leq &\  \kl^{-1}\qty(0, \frac{1}{|\bfi^c|} \ln \frac{2\sqrt{|\bfi^c|}}{\delta_{\bfi}^{\sigma}}) \label{eq:loose_kl}\\
=&\  \overline{\rm{Bin}}\qty(0,|\bfi^c|, \delta_{\bfi}^{\sigma})+ K(|\bfi^c|,\delta_{\bfi}^{\sigma})\,,\label{eq:constant_kl}
\end{align}
with 
$K(m,\delta) \ = \ \exp\qty(-\tfrac{1}{m} \ln \tfrac{1}{\delta}) - \exp\qty(-\tfrac{1}{m} \ln \tfrac{2\sqrt{m}}{\delta})$ and $\delta_{\bfi}^{\sigma} = \smqty(n \\ \m)^{-1} \zeta (\m)P_{M(\bfi)}(\sigma) \delta$.
\end{restatable}
The proofs of both Theorems \ref{thm:consistent_case_catoni} and \ref{thm:consistent_case_kl} can be found in \cref{app:consistent_case}.

Note that the constant $K(m,\delta)$ of Equation~\eqref{eq:constant_kl} tends to $0$ when $m$ tends to $\infty$ and is bounded by
\begin{equation*}
    0\ \leq\  K(m,\delta) \ \leq\  \frac{1}{m}\ln \frac{2\sqrt{m}}{\delta}.
\end{equation*}
With $\delta =0.01$, $K(m,\delta)$ is maximized at $m\approx7.35$, with $K(7.35, 0.01) \approx 0.11$. However, as depicted by the empirical results of Section~\ref{sec:experiments}, $\delta_{\bfi}^{\sigma}$ is orders of magnitude smaller than $0.01$ in practical situations.\footnote{Given a dataset of $10597$ datapoints, a compression set of size $92$ and $\delta=0.01$, we have $K(10505,10^{-234})\approx 0.0005$. This is indeed the difference that we observe in \cref{tab:binary_mnist_cnn} between the $\kl$ bound and the binomial bound for MNIST08.}

\section{EXPERIMENTS}\label{sec:experiments}
In this section, we show the versatility of our results by training different models using the P2L algorithm.\footnote{Our code is available at \url{https://github.com/GRAAL-Research/pick-to-learn}.}
In \cref{sec:binary_mnist}, we train neural networks on binary classification problems and compare our new results to the pre-existing sample compression ones. We empirically validate that our bounds are almost as tight as the binomial bound, all the while not suffering from the numerical optimization problem of \cref{thm:binom_tail} and being defined in the inconsistent case, where the P2L bound of \cref{thm:p2l} is undefined. In \cref{sec:mnist}, we train CNNs on the MNIST dataset and present generalization bounds on the (bounded) cross-entropy loss. As no previous sample-compression bound is defined for real-valued losses, we compare our result to a PAC-Bayesian theorem. In \cref{sec:regression}, we use P2L to train tree-based models on regression datasets and give generalization bounds on the root mean squared error (RMSE), an unbounded loss function, under the assumption that it is sub-Gaussian. Finally, in \cref{sec:amazon}, we fine-tune DistilBERT, a 66M parameters language model, on a review polarity classification problem. We obtain tight bounds simultaneously on the zero-one loss and the cross-entropy loss, demonstrating that our new theorem is independent of the number of parameters of the model.

Each experiment is run five times with different seeds. We report the mean and standard deviation of the metrics over these five repetitions. The datasets are separated into three parts: the training, validation and test set. When a dataset doesn't have a built-in test set, we create it using 10\% of the samples. Of the remaining samples, 10\% are used for the validation set size and 90\% for the training set. When computing the bounds, we use $\delta=0.01$. 
The hyperparameters used for the experiments can be found in \cref{app:experiments}. 
\begin{figure*}[t!] 
    \centering
    \begin{subfigure}[t]{0.5\textwidth}
        \centering
    \includegraphics[width=\textwidth]{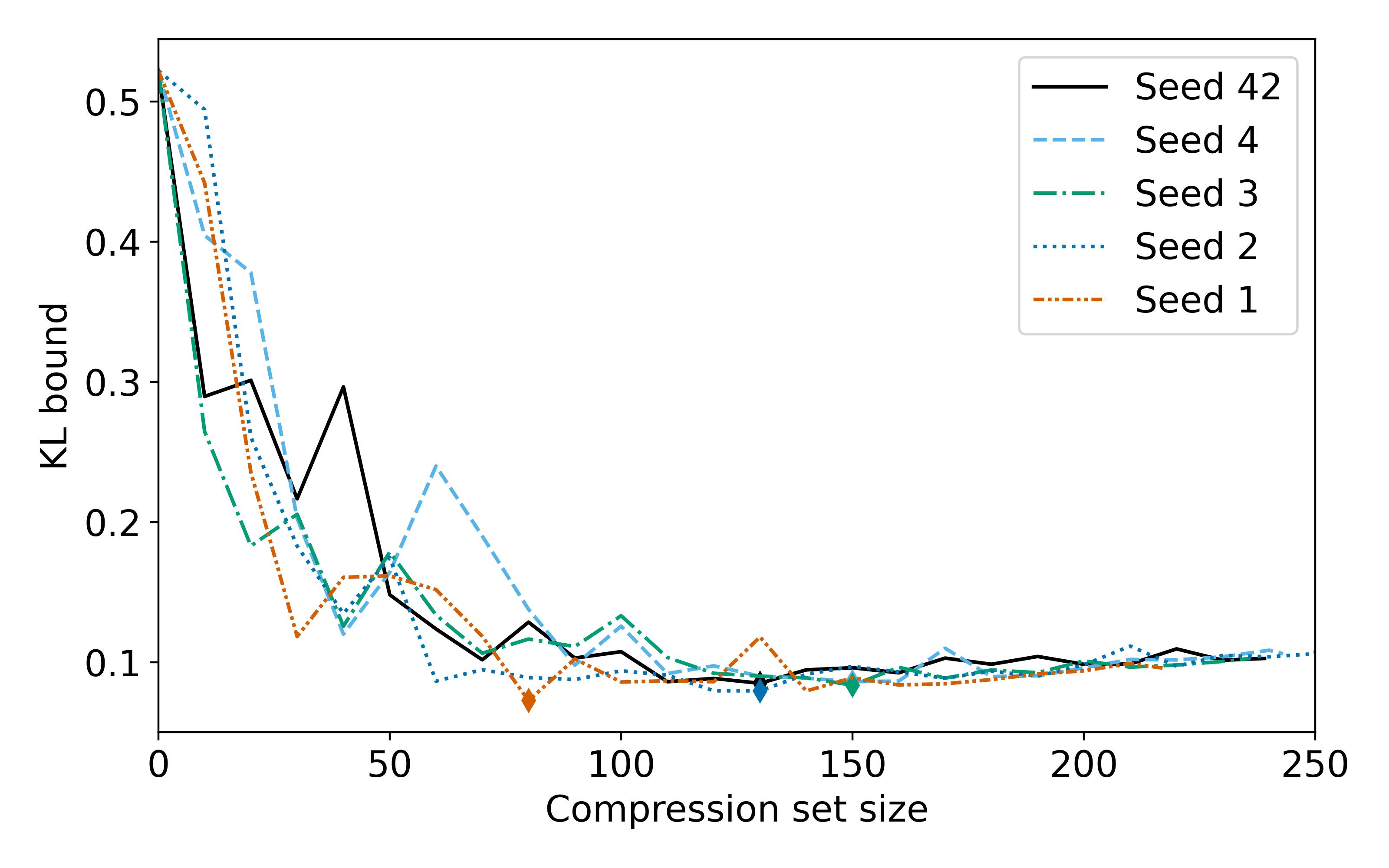}
        \caption{MNIST49}
    \end{subfigure}%
    ~ 
    \begin{subfigure}[t]{0.5\textwidth}
        \centering
        \includegraphics[width=\textwidth]{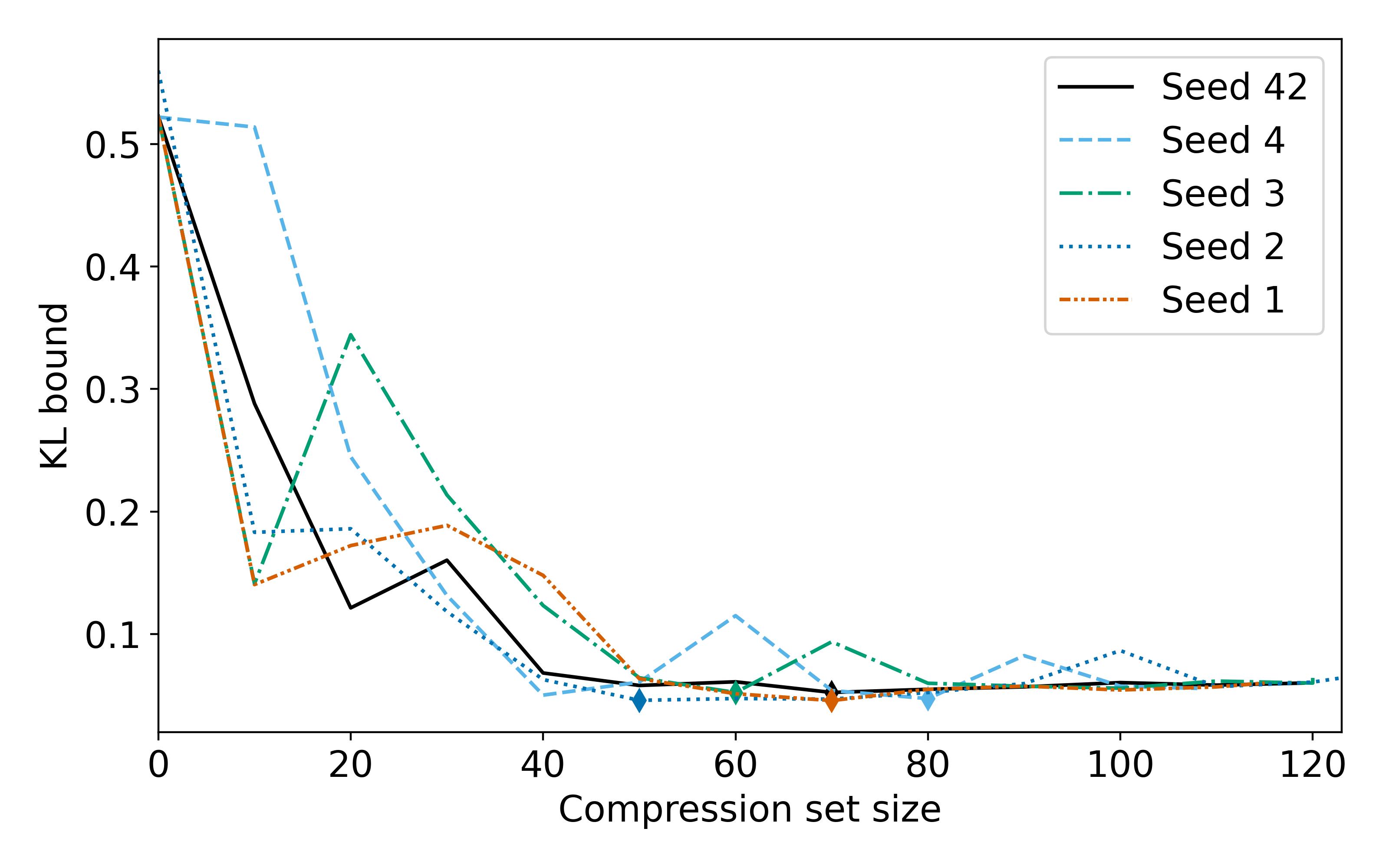}
        \caption{MNIST56}
    \end{subfigure}
    \hfill 
    \caption{Illustration of the behavior of the $\kl$ bound throughout P2L iterations for the five different seeds of the hyperparameter combination that achieved the minimal P2L bound on MNIST49 and MNIST56.  We mark the minimal $\kl$ bound for each seed with a diamond ($\blacklozenge$). The results for the other datasets can be found in \cref{fig:early_stop_appendix}.}\label{fig:early_stop}
\end{figure*}
\begin{table*}[!t]
  \caption{Cross-entropy loss achieved by the CNNs on MNIST. The results displayed obtained the smallest $\kl$ bound.}
  \label{tab:mnist_cross_entropy}
  \centering \small
  \begin{tabular}{cccccc}
    \toprule
    Learning algorithm & Train loss &	Test loss & $\kl$ bound &	$\m/n$ & Baseline test loss\\
    \midrule
P2L & 0.0014$\pm$0.0006  & 0.0467$\pm$0.0056 & 0.8279$\pm$0.0256 & 1.13$\pm$0.24& \multirow{2}{*}{0.0499$\pm$0.0042} \\ 
PBB & 0.0092$\pm$0.0005  & 0.0045$\pm$0.0004 & 0.0112$\pm$0.0005 & - & \\
    \bottomrule
  \end{tabular}
  \vspace{-0.3cm}
\end{table*}
\subsection{Binary MNIST}\label{sec:binary_mnist}
We create binary classification datasets by extracting pairs of digits from the MNIST dataset \citep{lecun1998_mnist}, e.g., selecting the datapoints labeled $0$ and $8$ to build the dataset MNIST08. We create five datasets: MNIST08, MNIST17, MNIST23, MNIST49 and MNIST56. Starting from randomly initialized neural networks, we train a MLP and a CNN using P2L on each dataset. More details are given in \cref{app:binary_mnist_problems}.

For all experiments in this section, we compute our proposed $\kl$ bound (\cref{corr:kl_bound}), the binomial approximation bound of \cite{shah_margin-sparsity_2005} (\cref{corr:binom_approx}, in appendix) and the P2L bound of \cite{paccagnan_pick_learn_2023} (\cref{thm:p2l}). We do not compute the binomial tail inversion of \cref{thm:binom_tail} as its optimization is very unstable. However, the binomial approximation is equivalent to \cref{thm:binom_tail} when $k=0$, which corresponds to the consistent case reached by the P2L algorithm.

\cref{tab:binary_mnist_cnn} presents results for the CNN in the consistent case. The results for the MLP can be found in \cref{app:binary_mnist_problems}. The error on the training set is zero for all predictors returned by P2L. The presented results achieve the tightest P2L bound for each dataset. The reported ``baseline test error'' corresponds to the results of a standard neural network optimized on the whole training set by stochastic gradient descent for 200 epochs or until the model achieves zero training errors; the selected hyperparmeters are the ones minimizing the validation error. For both CNN and MLP architectures, using P2L only incurs a slight increase of the test error compared to the baseline, whilst the model is trained on a fraction of the dataset, ranging from 0.7\%  to 3.4\%. Finally, even though the P2L bound is much tighter than the proposed $\kl$ bound, our result is much more general, as it holds for any real-valued loss functions and in the non-consistent case. Moreover, our bounds hold uniformly over all iterations of the models trained using P2L. After training, one can use any checkpoint of the model and still obtain a valid bound, which gives control over a trade-off between the training error, the generalization bound and the validation error. 

As for the inconsistent case, Figure~\ref{fig:early_stop} presents the behavior of the bound throughout the P2L iterations. We observe that the minimal $\kl$ bound value happens at about half the final number of iterations, leading to a smaller compression set and a tighter bound, as also reported in \cref{tab:binary_mnist_cnn_early_stopped}. Recall that the P2L bound cannot be computed in this case, as the model do not reach zero errors. In comparison to the previous consistent results (\cref{tab:binary_mnist_cnn}), the test error of \cref{tab:binary_mnist_cnn_early_stopped} are about twice as high as the fully trained model. However, the inconsistent models are trained on very small portions of the dataset, with the MNIST17 model being trained on 0.42\% of the dataset and still achieving a test error of 0.48\%. Finally, we observe that, in this setting, our new $\kl$ bound is much tighter than the binomial approximation of \cite{shah_margin-sparsity_2005}. 

\subsection{MNIST}\label{sec:mnist}
\begin{table*}[t]
\caption{Results for the decision forests trained using P2L. We report the RMSE achieved by the models with the smallest $\kl$ bound. The ratio $\m/n$ is presented in percents (\%).} \label{tab:decision_forests}
\begin{center}
\resizebox{2\columnwidth}{!}{%
\begin{tabular}{ccccccccc}
\toprule
Dataset & Train loss & Validation loss &	Test loss &	$\kl$ bound	&Linear bound	&$\m /n$&	Baseline test loss & $\ell^{\max}$ \\
\midrule
Powerplant & 6.11$\pm$0.89 &	6.23$\pm$0.70 &	6.31$\pm$0.95 &	13.69$\pm$0.27 & 15.92$\pm$0.47	& 0.51$\pm$0.17 &	3.59$\pm$0.13 & 90.6 \\
Infrared & 0.27$\pm$0.03 &	0.29$\pm$0.04 &	0.30$\pm$0.03	& 1.08$\pm$0.08 &  1.16$\pm$0.08 &	2.32$\pm$0.66 & 0.23$\pm$0.01 & 4.26 \\
Airfoil & 3.67$\pm$0.16 &	4.03$\pm$0.37 &	3.90$\pm$0.18 &	14.19$\pm$0.49 & 14.25$\pm$0.42 &	3.23$\pm$0.46	& 2.10$\pm$0.15 & 45.13 \\
Parkinson & 7.79$\pm$0.33 &	7.79$\pm$0.28 &	7.84$\pm$0.27 &	12.24$\pm$0.27 & 12.02$\pm$0.23 &	0.43$\pm$0.10 	 & 2.23$\pm$0.16 & 41.37 \\
Concrete & 8.18$\pm$0.91 &	8.70$\pm$1.00 &	8.48$\pm$1.41 &	31.68$\pm$1.63& 32.49$\pm$1.52 &	3.81$\pm$0.82  &	4.70$\pm$0.36 & 90.63\\ \bottomrule
\end{tabular}
}
\end{center}
\end{table*}
\begin{table*}[ht]
\caption{Results for the decision trees trained using P2L. We report the RMSE achieved by the models with the smallest $\kl$ bound. The ratio $\m/n$ is presented in percents (\%)} \label{tab:decision_trees}
\begin{center}
 \resizebox{2\columnwidth}{!}{%
\begin{tabular}{ccccccccc}
\toprule
Dataset & Train loss & Validation loss &	Test loss &	$\kl$ bound	&Linear bound	&$\m/n$&	Baseline test loss & $\ell^{\max}$ \\
\midrule
Powerplant & 11.66$\pm$3.00 &	11.83$\pm$3.12 &	12.00$\pm$3.04	&23.40$\pm$1.59 & 24.15$\pm$1.54 & 0.94$\pm$0.49 & 4.07$\pm$0.13 & 90.6\\
Infrared & 0.48$\pm$0.10	& 0.47$\pm$0.11 & 0.48$\pm$0.07 &	1.34$\pm$0.03 & 1.33$\pm$0.06 &	2.10$\pm$1.06	&0.27$\pm$0.03 & 4.26\\
Airfoil & 11.02$\pm$1.71 &	10.89$\pm$1.38 &	11.10$\pm$1.93 &	18.87$\pm$1.97 & 18.14$\pm$1.75 &	1.00$\pm$0.03 &	3.01$\pm$0.19 & 45.13\\
Parkinson & 13.93$\pm$2.83 &	13.84$\pm$2.76 &	13.99$\pm$2.92	&17.75$\pm$3.34 & 15.10$\pm$7.23	& 0.25$\pm$0.00 &	3.20$\pm$0.15 & 41.37\\
Concrete & 26.08$\pm$4.96 &	25.05$\pm$4.38 &	26.40$\pm$4.05 &	45.79$\pm$4.26 & 44.47$\pm$4.24 &	1.51$\pm$0.33	&6.22$\pm$0.91 & 90.63\\
\bottomrule
\end{tabular}
}
\end{center}
\vspace{-0.3cm}
\end{table*}
We now train convolutional neural networks composed of two convolutional layers and two fully connected layers. We pre-train the model using stochastic gradient descent on a subset of the dataset and then use P2L to fine-tune the model on the train set. The size of the pre-training subset is an hyperparameter. We use the same training setting as in \cref{sec:binary_mnist} and use the extension of P2L that adds multiple datapoints to the compression set at a time, with batch size $R=32$, as defined by Algorithm~2 of \cite{paccagnan_pick_learn_2023}. For comparison, we also train probabilistic neural networks (PNN) using the PAC-Bayes with Backprop (PBB) approach of \cite{perez2021tighter}, which optimize a PAC-Bayesian $\kl$ bound (\cref{thm:pbb} in appendix). 

For both our new sample-compression bounds and the PAC-Bayesian bound of \cite{perez2021tighter}, we compute the bounds on the zero-one loss and on a bounded version of the cross-entropy loss (see \cref{app:mnist_problems}). The probabilities outputted by the neural networks are restricted to be greater than $10^{-5}$, effectively bounding the cross-entropy by $-\ln(10^{-5}) \approx 11.51$.

\cref{tab:mnist_cross_entropy} reports the bound values for the bounded cross-entropy loss
(see \cref{tab:mnist_risk_01} for classification error).
We observe that the PBB algorithm gives a tighter generalization bound than the one of P2L. This gap can be explained by the fact that PBB jointly optimizes the train error and the $\KL$ divergence, whilst we have almost no control on the minimization of the bound. Indeed, the heuristic of the P2L algorithm, which is to choose the datapoints over which the model incurs the greatest losses, doesn't give control on the trade-off between the decrease of the error and the increase of the complexity term. Moreover, for a large dataset, the binomial coefficient increases rapidly when the compression set size increases. However, using our bounds with the P2L algorithm has multiple advantages over the PBB algorithm. First of all, PBB needs to train twice as many parameters, as it fits both the mean and standard deviation of the distributions over the parameters. Secondly, computing the PAC-Bayesian bound necessitates a step of Monte Carlo sampling to determine the average error of the model. For 5000 steps of Monte Carlo sampling, the prediction over the dataset is computed 5000 times, instead of only once with P2L. Finally, our bound doesn't take into account the number of parameters of the model, whilst the $\KL$ divergence in \cref{thm:pbb} is a sum of the $\KL$ divergence of the distribution of each parameter of the model.
\subsection{Regression with tree-based models}\label{sec:regression}
\begin{table*}[t]
\caption{Results for the amazon polarity dataset. The results displayed for P2L obtained the lowest $\kl$ bound on the error, whilst the baseline was chosen by the lowest validation error. The ratio $\m/n$ is presented in percents (\%).} \label{tab:amazon_polarity}
\begin{center}
\small
\resizebox{2\columnwidth}{!}{%
\begin{tabular}{ccccccccccc}
\toprule
\multirow{2}{4em}{Learning algorithm} & \multicolumn{4}{c}{Error (\%)} &\multicolumn{3}{c}{Cross-entropy loss} \\ \cmidrule(l{2pt}r{2pt}){2-5} \cmidrule(l{2pt}r{2pt}){6-8} &  Train & Test & $\kl$ bound & Binomial bound & Train &  Test & $\kl$ bound & $\m/n$\\
\midrule
P2L & 4.73$\pm$1.09 & 5.60$\pm$1.19 & 13.91$\pm$2.73 & 21.85$\pm$3.29 & 0.1199$\pm$0.0118  & 0.1478$\pm$0.0182 & 0.8594$\pm$0.1622 & 0.79$\pm$0.23\\
Baseline & 3.11$\pm$0.02 & 4.19$\pm$0.00 & - & - & 0.0912$\pm$0.0010 & 0.1158$\pm$0.0002 & - & -\\
\bottomrule
\end{tabular}
}
\end{center}\vspace{-0.3cm}
\end{table*}
In order to show the wide applicability of our bounds, we train decision forests on regression problems: Powerplant \citep{Tfekci2014PredictionOF}, Infrared \citep{wang2021infrared}, Airfoil \citep{brooks1989airfoil}, Parkinson \citep{tsanas2009accurate} and Concrete \citep{yeh1998modeling}. These datasets range from a training set size of 827 to 7751 and range from a number of features of 4 to 33. 
To the best of our knowledge, no sample compression bounds exist for this setting. Previous results were presented for linear regression under an $\ell_p$ loss \citep{hanneke_agnostic_2018,attias2023optimal,attias2024sample} or for boosting real-valued learner in a binary classification setting \citep{hanneke_sample_2019}, none of which are equivalent to our setting. We adapt the P2L algorithm to this regression problem (see Algorithm~\ref{alg:p2l_trees} in appendix), which differs from the original one, designed only for classification problems where zero training error is achievable (consistent case). At each P2L iteration, we add a single datapoint to the compression set in order to train the forest. The selected datapoint is the one with the largest root mean squared error (RMSE). Then, the trees are retrained from scratch on the compression set. As the minimal RMSE that can be achieved is dependent on the dataset, setting a predetermined threshold is not a suitable stopping criterion. Thus, we train the model until the validation loss has not decreased for a given number of iterations. To compute the bounds, we need the loss to be either bounded or sub-Gaussian. As tree-based models predict the mean of the targets of each datapoint assigned to a leaf, their outputs are bounded by the extrema of the data. 
To compute the $\kl$ bound, we assume that the target space is bounded by the maximum value of the loss $\ell^{\max}$ reported in \cref{tab:decision_forests}. To compute the linear bound, we assume that the loss is sub-Gaussian. We discuss in more details these assumptions and the way of defining the extrema in \cref{app:regression_problems}.

\cref{tab:decision_forests} contains the results of the models selected based on the smallest $\kl$ bounds. We observe that the models trained with P2L are able to obtain competitive results with respect to the test error of standard random forests trained on the whole dataset. We report latter results in the column ``baseline test loss'', where the models are chosen by their validation loss. As the values of the bounds are much smaller than $\ell^{\max}$, we conclude that our bounds are tight and non-vacuous. 
The generalization guarantees given by the bounds relying on the linear function are competitive to the ones relying on the $\kl$; they are even tighter on the Parkinson dataset.

Following a similar experiment setting, we trained regression trees with P2L. As we can see in \cref{tab:decision_trees}, training trees using P2L leads to underfitted models that are not competitive with respect to the baseline. Indeed, as the trees are trained on a small subset of the data, they are restricted to be less complex than trees trained on the whole dataset. When selecting the models by the smallest validation loss (see \cref{tab:decision_trees_validation_loss}), we observe that the models achieve better performance, as the compression sets are much larger, but also suffer from worsened bounds. %
\subsection{Amazon polarity}\label{sec:amazon}
Finally, we train DistilBERT \citep{sanh2019distilbert} on the Amazon reviews polarity dataset \citep{zhang2015character}. Using P2L, we fine-tune the pretrained language model on 10\% of the dataset, for a total of 360k datapoints, and evaluate the model on the test set, which comprises 400k datapoints. We pre-train the model on half of the training dataset and then use P2L on the other half of the training set. We add 32 datapoints at a time in the compression set and early stop the training of the model if its validation loss has not decreased for 20 epochs. 
In this experiment, we study our new $\kl$ bound on the zero-one loss and on the bounded cross-entropy loss. Moreover, we compute the binomial approximation bound of \cref{corr:binom_approx}. The P2L bound (\cref{thm:p2l}) is inapplicable in this setting, as the model doesn't reach zero errors. The PAC-Bayesian bound of \cref{thm:pbb} could be computed on both metrics, but it would necessitate to train 132M parameters (twice the number of parameters of DistilBERT). Many new generalization bounds and approaches were presented for very large models \citep{lotfi2023non, lotfi2025unlocking, zekri2024largelanguagemodelsmarkov, su2024missionimpossiblestatisticalperspective}, such as large language models. However, most approaches are not suited for classification and regression, as they are derived for language modeling objectives.

From the results displayed by \cref{tab:amazon_polarity}, we first observe that training the model using P2L only incurs a loss of about a percent for the train, validation and test error compared to the baseline. It achieves this error whilst being trained on about 0.8\% of the dataset, as the compression set size is roughly 1138 datapoints and the training set size is 144k. Both for the error and the cross-entropy loss, the bound is tight and non-vacuous. Our bound is much tighter than the binomial approximation bound, with a certificate of $13.91\%$ for a train error of $4.73\%$. Thus, despite the 66M parameters of DistilBERT, we are able to obtain tight generalization guarantees by simply changing the training loop of the model for the P2L scheme. 

\section{CONCLUSION}
We proposed novel generalization bounds for real-valued losses and sample-compressed predictors. These bounds leverage the comparator functions studied in the PAC-Bayes theory. We provide results for bounded and unbounded losses, under different assumptions. We empirically verified the tightness of the proposed bounds, showing that it is almost as tight as the binomial tail inversion, which, however, holds only for a less general setting. We trained neural networks with 66M parameters and obtained tight guarantees, without suffering from the cost of the number of parameters. This highlights an important asset of the sample compression framework: Two models achieving the same empirical loss using the same amount of datapoints (compression set size) share the same guarantees (bound value), regardless of their size in terms of the number of trainable parameters.

In future works, we could leverage the possibility of having a message in the compression scheme, by training models such as the set covering machine \citep{shah_margin-sparsity_2005} or decision trees \citep{shah_sample_2007}, which both use binary sequences to specify how to reconstruct the model. Finally, although P2L is generally able to train good performing models, it is unclear that its sample selection heuristic is optimal for neural networks. Trying different heuristics, e.g., that optimize for sample diversity, could lead to further improvements.
\clearpage

\subsubsection*{Acknowledgements}
We would like to thank the anonymous reviewers for their helpful comments. We also wish to thank Benjamin Leblanc and Sokhna Diarra Mbacke for the insightful discussions and their help proof-reading the manuscript.

Mathieu Bazinet is supported by a FRQNT B2X scholarship (343192). Pascal Germain is supported by the Canada CIFAR AI Chair Program and the NSERC Discovery grant RGPIN-2020-07223.

\bibliographystyle{apalike}
\bibliography{bibliography}

\section*{Checklist}

 \begin{enumerate}
 \item For all models and algorithms presented, check if you include:
 \begin{enumerate}
   \item A clear description of the mathematical setting, assumptions, algorithm, and/or model. \textbf{Yes, in \cref{sec:background}.}
   \item An analysis of the properties and complexity (time, space, sample size) of any algorithm. \textbf{Not applicable.}
   \item (Optional) Anonymized source code, with specification of all dependencies, including external libraries. \textbf{Yes, the code can be found here : \url{https://github.com/GRAAL-Research/pick-to-learn}.}
 \end{enumerate}

 \item For any theoretical claim, check if you include:
 \begin{enumerate}
   \item Statements of the full set of assumptions of all theoretical results. \textbf{Yes, the only assumption is that the data is \emph{i.i.d.} and is mentioned in our setting in \cref{sec:background}.}
   \item Complete proofs of all theoretical results. \textbf{Yes, in \cref{app:proofs}.}
   \item Clear explanations of any assumptions. \textbf{The only assumption is that the data is \emph{i.i.d.}, which is the classical setting in the literature. }
 \end{enumerate}

 \item For all figures and tables that present empirical results, check if you include:
 \begin{enumerate}
   \item The code, data, and instructions needed to reproduce the main experimental results (either in the supplemental material or as a URL). \textbf{All the code and instruction can be found in the repository. The data is imported directly when the code is executed.}
   \item All the training details (e.g., data splits, hyperparameters, how they were chosen). \textbf{All the training details can be found in \cref{app:experiments} and the code.}
 \item A clear definition of the specific measure or statistics and error bars (e.g., with respect to the random seed after running experiments multiple times). \textbf{All metrics are defined in either \cref{sec:experiments} or \cref{app:experiments}. We present the mean and standard deviation over five seeds.}
 \item A description of the computing infrastructure used. (e.g., type of GPUs, internal cluster, or cloud provider). \textbf{Yes, in \cref{app:experiments}.}
 \end{enumerate}

 \item If you are using existing assets (e.g., code, data, models) or curating/releasing new assets, check if you include:
 \begin{enumerate}
   \item Citations of the creator If your work uses existing assets. \textbf{Yes.}
   \item The license information of the assets, if applicable. \textbf{Yes, in \cref{app:experiments}.}
   \item New assets either in the supplemental material or as a URL, if applicable. \textbf{Yes, the code can be found here : \url{https://github.com/GRAAL-Research/pick-to-learn}.}
   \item Information about consent from data providers/curators. \textbf{Not applicable.}
   \item Discussion of sensible content if applicable, e.g., personally identifiable information or offensive content. \textbf{Not applicable.}
 \end{enumerate}

 \item If you used crowdsourcing or conducted research with human subjects, check if you include:
 \begin{enumerate}
   \item The full text of instructions given to participants and screenshots. \textbf{Not applicable}
   \item Descriptions of potential participant risks, with links to Institutional Review Board (IRB) approvals if applicable. \textbf{Not applicable}
   \item The estimated hourly wage paid to participants and the total amount spent on participant compensation. \textbf{Not applicable}
 \end{enumerate}

 \end{enumerate}
 
\appendix
\onecolumn
\aistatstitle{Sample compression unleashed : \\
Supplementary Materials}
\section{EXPERIMENTS} \label{app:experiments}

\paragraph{Devices.}
The experiments were run on two different devices. The experiments with PBB algorithm and the regression datasets were run on Python 3.12.2 on a computer with a NVIDIA GeForce RTX 4090. The experiments on MNIST were run on Python 3.12.3 on a computer with a NVIDIA GeForce RTX 2080 Ti. 

\paragraph{Librairies.}
The libraries used for each environment can be found with the code. Notably, we use PyTorch \citep{Ansel_PyTorch_2_Faster_2024} (BSD 3-Clause License), Lightning \citep{Falcon_PyTorch_Lightning_2019} (Apache 2.0 license), Weight and Biases \citep{wandb} (MIT License), Scikit-Learn \citep{scikit-learn} (BSD 3-Clause License), NumPy \citep{harris2020numpy} (NumPy license) and Transformer \citep{wolf2020huggingface} (Apache 2.0 license). For all experiments, we run the code with the following seeds : $\{1,2,3,4,42\}$.

\paragraph{Datasets.}
For the classification problems, we use the MNIST dataset \citep{lecun1998_mnist} (MIT License) and the amazon polarity dataset \citep{zhang2015character} (Apache 2.0 License). All MNIST derived-dataset are composed of 784 real-valued features. For the multi-class classification problems on MNIST, we denote MNIST ($p$\%) to say that we pre-train the model on $p$\% of the data, where $p$ is a hyperparameter. For the Amazon polarity dataset, we chose 10\% of the dataset to create a 360k datapoints dataset. We then use 50\% to pre-train the model and split the rest into a training and validation set. The datapoints are textual reviews and the labels are binary. The description of the dataset are presented in \cref{tab:classification_dataset_infos}.

\begin{table}[ht]
\caption{Description of the datasets used for classification problems.} \label{tab:classification_dataset_infos}
\begin{center}
\begin{tabular}{cccccccc}
\toprule
Dataset & Pretrain set size & Train set size & Validation set size &	Test set size\\
\midrule
Amazon Polarity & 180000 & 144000 & 36000 & 400000\\  
MNIST (10\%) & 6000 & 48000 & 6000 & 10000 \\
MNIST (20\%) & 12000& 42000 & 6000 & 10000 \\
MNIST (50\%) & 30000 & 24000 & 6000 & 10000 \\
MNIST08 & 0 & 10597 & 1177 & 1954 \\
MNIST17 & 0 & 11707 & 1300 & 2163 \\
MNIST23 & 0 & 10881 & 1208 & 2042 \\
MNIST49 & 0 & 10612 & 1179 & 1991\\
MNIST56 & 0 & 10206 & 1133 & 1850 \\
\bottomrule
\end{tabular}
\end{center}
\end{table}

For the regression problems, we train our models on five datasets : the \emph{Combined Cycle Power Plant} \citep{Tfekci2014PredictionOF, combined_cycle_power_plant_294}, the \emph{Infrared Thermography Temperature} \citep{wang2021infrared, wang2023facial}, the \emph{Airfoil Self-Noise} \citep{brooks1989airfoil, airfoil_self-noise_291}, the \emph{Parkinsons Telemonitoring} \citep{tsanas2009accurate, parkinsons_telemonitoring_189} and the \emph{Concrete Compressive Strength} \citep{yeh1998modeling, concrete_compressive_strength_165}. The descriptions of the dataset are presented in \cref{tab:regression_dataset_infos}. All datasets were chosen from the UCI dataset repository. Powerplant, Airfoil, Parkinson and Concrete are under the CC-BY 4.0 license. The Infrared dataset is under the CC0 license. 
\begin{table}[ht]
\caption{Description of the datasets used for regression problems.} \label{tab:regression_dataset_infos}
\begin{center}
\begin{tabular}{ccccccc}
\toprule
Dataset & Train set size & Validation set size &	Test set size&	Number of features	 \\
\midrule
Powerplant & 7751 & 861 & 956 & 4  \\
Infrared & 827 & 91 & 102 & 33 \\
Airfoil & 1218 & 135 & 150 & 5 \\
Parkinson & 4760 & 528 & 587 & 19 \\
Concrete & 835 & 92 & 103 & 8 \\
\bottomrule
\end{tabular}
\end{center}
\end{table}

\subsection{Hyperparameter grids}\label{sec:hyperparameters_grid}

In this section, we present the hyperparameter grids for all the experiments.

In all experiments, we use $\delta=0.01$ and a batch size of $64$. After each iteration of P2L, we train the model for 200 epochs or until the validation loss has not improved for three epochs.

\subsubsection{Binary MNIST problems}\label{app:binary_mnist_problems}

For the binary MNIST problems, we used the following hyperparameters for both MLP and CNN architectures. 
\begin{itemize}
    \item Dropout probability : $\{0.1, 0.2\}$
    \item Training learning rate : $\{10^{-2}, 10^{-3}, 5\times 10^{-3}, 10^{-4}\}$
\end{itemize}

The MLP is composed of three hidden fully connected layers of 600 neurons and the CNN is composed of two convolutional layers and two fully connected layers. We use ReLU activations \citep{glorot2011deep}, dropout layers \citep{dropout2014} and the Adam optimizer \citep{kingma2014adam} with the default parameters $\beta=(0.9, 0.999)$.

At each iteration, the P2L algorithm adds one datapoint to the compression set. We use $h_0$ a randomly initialized neural network.

For the baselines, we train the same models with the same hyperparameters for 200 epochs or until the model achieves zero errors on the training set.

We present the results for the MLP, both trained fully using P2L and early-stopped, respectively in \cref{tab:binary_mnist_mlp} and in \cref{tab:binary_mnist_mlp_early_stopped}. Moreover, \cref{fig:early_stop_appendix} displays the results not present in \cref{fig:early_stop}.

\begin{table}[h]
\caption{Results for the MLPs  trained using P2L on the binary MNIST problems. The results displayed obtained the tightest P2L bound. All metrics presented are in percents (\%).} \label{tab:binary_mnist_mlp}
\begin{center}

\resizebox{1\columnwidth}{!}{%
\begin{tabular}{cccccccc} \toprule
Dataset & Validation error &	Test error&	$\kl$ bound	&Binomial bound	&P2L bound	&$\m/n$&	Baseline test error \\
\midrule
MNIST08 &	0.41$\pm$0.14&	0.40$\pm$0.08&	6.56$\pm$0.30&	6.51$\pm$0.30&	1.42$\pm$0.08&	1.21$\pm$0.07&	0.34$\pm$0.07 \\
MNIST17&	0.37$\pm$0.14&	0.47$\pm$0.17&	4.93$\pm$0.27&	4.89$\pm$0.27&	1.01$\pm$0.07&	0.85$\pm$0.06&	0.33$\pm$0.09 \\ 
MNIST23&	0.87$\pm$0.24&	0.58$\pm$0.12&	12.21$\pm$0.29&	12.17$\pm$0.29	& 3.06$\pm$0.09&	2.73$\pm$0.09&	0.36$\pm$0.14 \\ 
MNIST49&	1.19$\pm$0.33&	1.04$\pm$0.10&	14.41$\pm$0.05&	14.37$\pm$0.05&	3.78$\pm$0.02&	3.41$\pm$0.02&	0.96$\pm$0.15 \\ 
MNIST56&	0.68$\pm$0.17&	0.65$\pm$0.05&	10.35$\pm$0.31&	10.30$\pm$0.31	&2.48$\pm$0.09&	2.18$\pm$0.09&	0.59$\pm$0.01 \\ \bottomrule
\end{tabular}
}
\end{center}
\end{table}
\begin{table}[h]
\caption{ Results for the MLPs trained using P2L on the binary MNIST problems and stopped at the iteration with the minimum $\kl$ bound. The results displayed obtained the tightest $\kl$ bound. All metrics presented are in percents~(\%).} \label{tab:binary_mnist_mlp_early_stopped}
\begin{center}
\resizebox{1\columnwidth}{!}{%
\begin{tabular}{cccccccc}
\toprule
Dataset &  Validation error &	Test error&	$\kl$ bound	&Binomial bound	& Train error & $\m / n$&	Baseline test error \\
\midrule
MNIST08 & 1.11$\pm$0.52 & 1.04$\pm$0.67 & 5.46$\pm$0.53 & 7.77$\pm$1.64 & 0.85$\pm$0.71 & 0.56$\pm$0.32 & 0.34$\pm$0.07\\
MNIST17&  0.88$\pm$0.39 & 0.80$\pm$0.29 & 4.02$\pm$0.36 & 5.49$\pm$0.77 & 0.50$\pm$0.26 & 0.38$\pm$0.13	& 0.33$\pm$0.09 \\ 
MNIST23&  1.93$\pm$0.49 & 1.59$\pm$0.43 & 10.86$\pm$0.19 & 13.23$\pm$0.74 & 1.27$\pm$0.41 & 1.34$\pm$0.24	& 0.36$\pm$0.14 \\ 
MNIST49&  2.28$\pm$0.53 & 2.07$\pm$0.58 & 13.14$\pm$0.32 & 15.08$\pm$0.99 & 1.22$\pm$0.47 & 1.90$\pm$0.29 & 0.96$\pm$0.15 \\ 
MNIST56& 1.97$\pm$0.53 & 1.88$\pm$0.44 & 8.85$\pm$0.58 & 11.78$\pm$1.44 & 1.38$\pm$0.61 & 0.90$\pm$0.27 & 0.59$\pm$0.01\\ \bottomrule
\end{tabular}
}
\end{center}
\end{table}

\begin{figure*}[t!]
    \centering
    \begin{subfigure}[t]{0.5\textwidth}
        \centering
    \includegraphics[width=\textwidth]{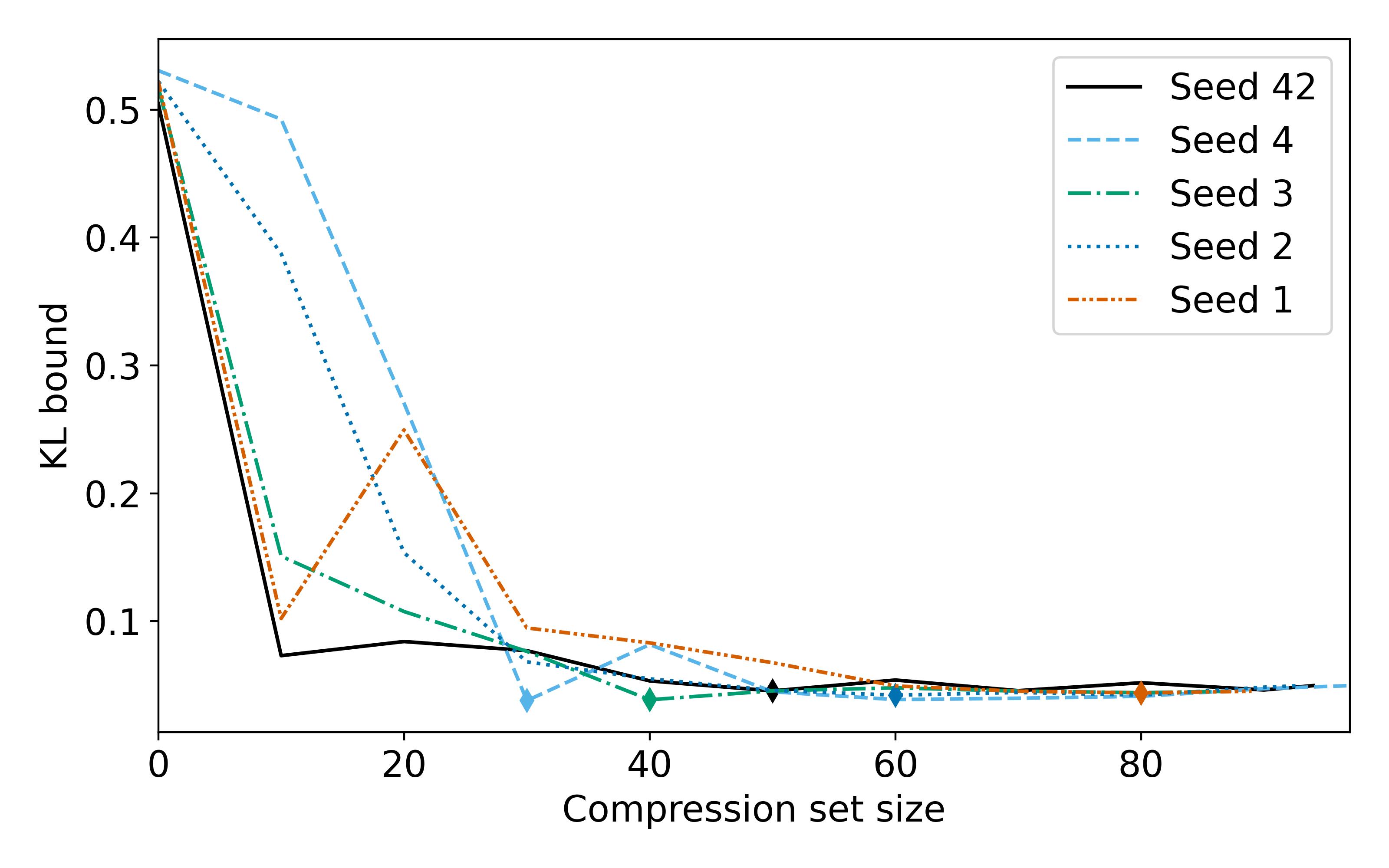}
        \caption{MNIST08}
    \end{subfigure}%
    ~ 
    \begin{subfigure}[t]{0.5\textwidth}
        \centering
        \includegraphics[width=\textwidth]{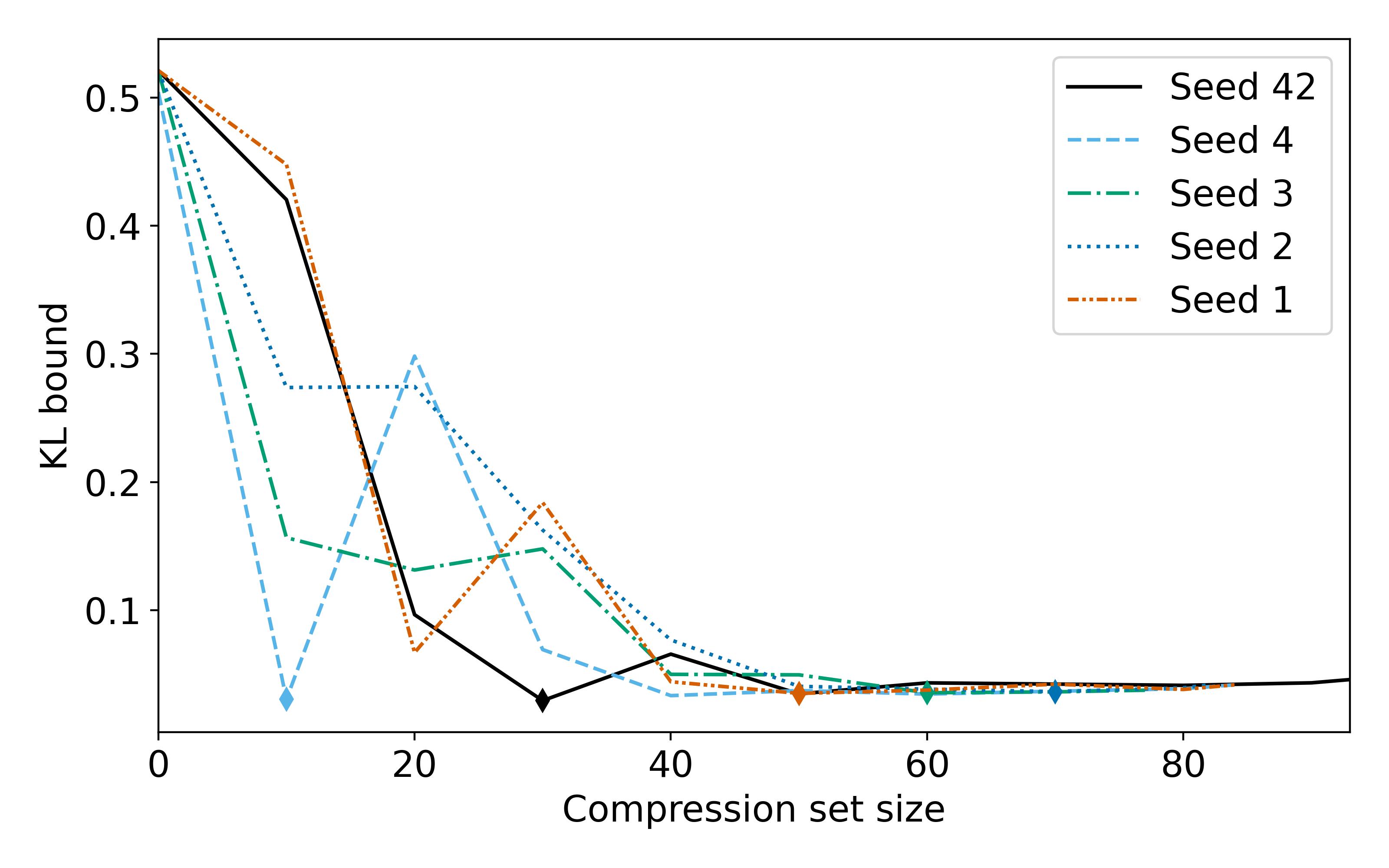}
        \caption{MNIST17}
    \end{subfigure}
    \hfill 
    \\ 
        \begin{subfigure}[t]{0.5\textwidth}
        \centering
        \includegraphics[width=\textwidth]{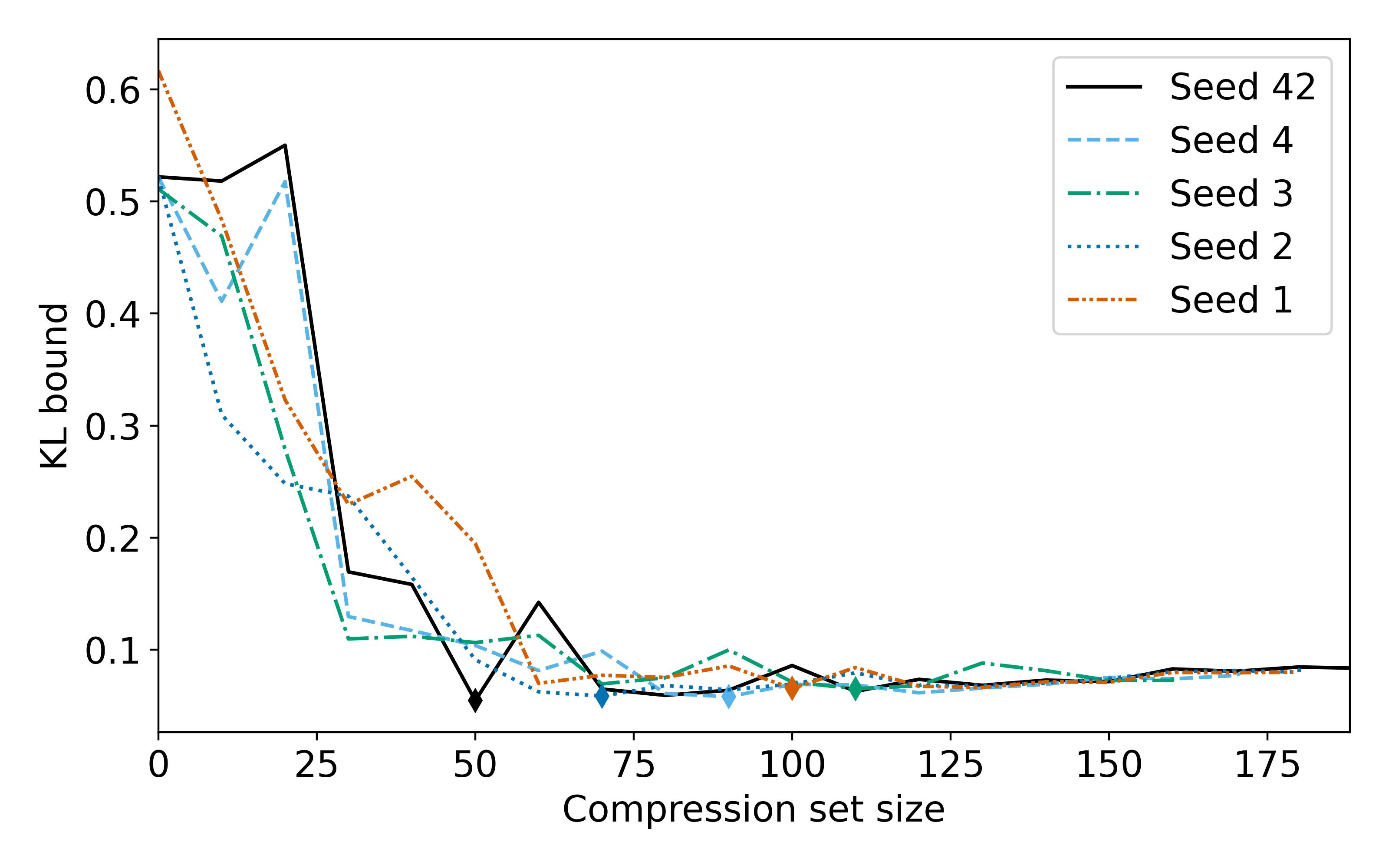}
        \caption{MNIST23}
    \end{subfigure}
    \caption{Illustration of the behavior of the $\kl$ bound throughout P2L iterations for the five different random seed initializations sharing the hyperparameter combination that achieved the minimal averaged P2L bound.  The diamonds ($\blacklozenge$) mark the minimal $\kl$ bound for each run. } \label{fig:early_stop_appendix}
\end{figure*}

\subsubsection{MNIST problems}\label{app:mnist_problems}
\begin{table*}[t]
  \caption{Classification risk achieved by the CNNs on MNIST. The results displayed obtained the smallest $\kl$ bound. All metrics are presented in percents (\%).}
  \label{tab:mnist_risk_01}
  \centering
  \begin{tabular}{lcccccc}
    \toprule
    Learning algorithm & Train error &	Test Error &	$\kl$ bound &	Binomial bound &	$\m/n$ & Baseline test error\\
    \midrule
P2L & 0.0$\pm$0.0 &	1.14$\pm$0.07 & 7.11$\pm$0.22 & 7.08$\pm$0.22 &   1.15$\pm$0.34 & \multirow{2}{*}{1.08$\pm$0.09}\\
PBB & 1.67 $\pm$ 0.07  & 1.05$\pm$0.05 & 1.94$\pm$0.07 & - & - & \\
    \bottomrule
  \end{tabular}
\end{table*}

We train a convolutional neural network over the 10-class MNIST dataset with the following hyperparameters.

\begin{multicols}{2}
\begin{itemize}
    \item Size of pretraining set : $\{10\%, 20\%, 50\%\}$
    \item Pretraining epochs : $\{50, 100\}$
    \item Pretraining learning rate : $\{10^{-2}, 10^{-3}, 10^{-4}\}$
    \item Dropout probability : $\{0.1, 0.2\}$
    \item Training learning rate : $\{10^{-2}, 5\times 10^{-3}, 10^{-4}\}$
\end{itemize}
\end{multicols}

At each iteration, the P2L algorithm adds 32 datapoints to the compression set. We use $h_0$ a neural network pre-trained on $p\%$ of the dataset (see size of pretraining set for the different values of $p$).

To compute bounds for the cross-entropy loss, we clamp the log-probabilities to be greater or equal than $\ln(10^{-5})$ \citep{perez2021tighter, dziugaite_roy_2018}, as follows : 
\begin{equation*}
    \ell(h, \bx, y) =  - \max \qty( \ln\qty(10^{-5}),\, \ln\qty(\frac{\exp(h(\bx)_{y})}{\sum_{c=1}^C \exp(h(\bx)_c)})),
\end{equation*}
where $h(\bx) = \qty(h(\bx)_1, \ldots, h(\bx)_C)$ is the output of the neural network and $C$ is the number of classes. The loss then takes values in $[0, -\ln (10^{-5})]$. We use the same bounded cross-entropy loss for the following experiments.

For the baseline, we train the same model with the same hyperparameters for 200 epochs or until the model achieves zero errors on the training set.

For the PAC-Bayes with Backprop (PBB) algorithm, we used the code of the GitHub repository provided alongside \cite{perez2021tighter}, with the same hyperparameter grid proposed by the authors,\footnote{\url{https://github.com/mperezortiz/PBB}} except for the dropout rate, which we kept the same as the other experiments:
\begin{multicols}{2}
\begin{itemize}
    \item Scale parameter of the prior distribution : $\{0.1, 0.05, 0.04, 0.03, 0.02, 0.01, 0.005\}$
    \item Training learning rate : $\{10^{-3}, 5\times 10^{-3}, 10^{-2}\}$
    \item Pre-training learning rate : $\{10^{-3}, 5\times 10^{-3}, 10^{-2}\}$
    \item Momentum : $\{0.95, 0.99\}$
    \item Dropout probability : $\{0.1, 0.2\}$
\end{itemize}
\end{multicols}
 We fixed $\delta = \delta' = 0.01$ to compute \cref{thm:pbb}, and  performed $m=5000$ Monte Carlo sampling steps instead of using the value $m=150000$ found in the code, as it takes several hours to run.
 
\subsubsection{Regression problems}\label{app:regression_problems}

We trained decision trees and forests on the datasets, using P2L to train the models on one datapoint at a time. We trained the models until their validation loss hasn't decreased for 10 or 20 epochs. We summarize this idea in Algorithm~\ref{alg:p2l_trees}. We denote the RMSE as $\ell^{\textrm{RMSE}}(h, \bx,y) = \sqrt{(h(\bx)-y)^2}$ and the empirical risk on the dataset
\begin{equation*}
\calL^{\textrm{RMSE}}_{S}(h) = \sqrt{\frac{1}{n} \sum_{i=1}^n (h(\bx_i) - y_i)^2}.  
\end{equation*}
For tree-based models, we chose $h_0$ to simply output zeroes for all entries. We use $\texttt{COUNTER}$ and $\hatL_{\textrm{BEST}}$ as variables to stop the training when the loss hasn't decreased for $T$ epochs.

\SetKwInOut{Init}{Initialize}
\SetKwInOut{Input}{Input}
\SetAlgoLined
\DontPrintSemicolon
\begin{algorithm}[t]
\caption{Pick-To-Learn for regression problems}\label{alg:p2l_trees}
\Input{$T$, the number of look-ahead iterations to perform before stopping.}
\Init{$S_{\bfi} \leftarrow \emptyset $.}
\Init{$h_{\bfi} \leftarrow h_0$.}
\Init{$\calL_{\textrm{BEST}} \leftarrow \infty.$}
\Init{$\texttt{COUNTER} \leftarrow 0$.}
\Init{$(\overline{\bx}, \overline{y}) \leftarrow \argmax_{(\bx,y) \in S} \ell^{\textrm{RMSE}}(h_0, \bx, y)$}
\While{$\emph{\texttt{COUNTER}} \leq T$}{
$S_{\bfi} \leftarrow S_{\bfi} \cup \{(\overline{\bx}, \overline{y})\}$\;
$h_{\bfi} \leftarrow A(S_{\bfi})$\;
$(\overline{\bx}, \overline{y}) \leftarrow \argmax_{(\bx,y) \in S_{\bfi^c}} \ell^{\textrm{RMSE}}(h_{\bfi}, \bx, y)$\;
\eIf{$\calL^{\textrm{RMSE}}_{S_{\bfi^c}}(h_{\bfi}) < \calL_{\textrm{BEST}}$}{
$\calL_{\textrm{BEST}} \leftarrow \calL^{\textrm{RMSE}}_{S_{\bfi^c}}(h_{\bfi})$\;
$\texttt{COUNTER} \leftarrow 0$\;
}{
$\texttt{COUNTER} \leftarrow \texttt{COUNTER} + 1$
}
}
\Return $h_{\bfi}$
\end{algorithm}

We now present the hyperparameter grid.
\begin{multicols}{2}
\begin{itemize}
    \item Maximum depth of the trees : $\{5,10\}$
    \item Minimum samples to split : $\{2,3,4\}$
    \item Minimum samples to create a leaf : $\{1,2,3\}$
    \item Cost-Complexity pruning parameter : $\{0.0, 0.05, 0.1, 0.2, 0.5, 1, 2\}$
    \item Number of epochs before early stopping: $\{10, 20\}$
\end{itemize}
\end{multicols}

For the decision forests, we choose the number of estimators in $\{50, 100\}$. For the baselines, we train the same model with the same hyperparameters on the whole dataset. We present the results for the forests in \cref{tab:decision_forests} and \cref{tab:decision_forests_validation_loss}.

\begin{table}[ht]
\caption{Results for the decision forests trained using P2L. We report the RMSE achieved by the models with the smallest validation loss. The ratio $\m/n$ is presented in percents (\%).} \label{tab:decision_forests_validation_loss}
\begin{center}
 \resizebox{\columnwidth}{!}{%
\begin{tabular}{ccccccccc}
\toprule
Dataset & Train loss & Validation loss &	Test loss &	$\kl$ bound	&Linear bound	&$\m/n$&	Baseline test loss & $\ell^{\max}$ \\
\midrule
Powerplant & 4.15$\pm$0.17 & 4.47$\pm$0.21 & 4.51$\pm$0.21 & 18.66$\pm$1.71 & 25.74$\pm$2.06 & 2.13$\pm$0.53 & 3.59$\pm$0.13 & 90.6\\
Infrared & 0.21$\pm$0.01 & 0.23$\pm$0.02 & 0.26$\pm$0.01 & 1.19$\pm$0.06 & 1.48$\pm$0.06 & 3.60$\pm$0.46 &0.23$\pm$0.01 & 4.26\\
Airfoil & 1.78$\pm$0.21 & 2.28$\pm$0.29 & 2.37$\pm$0.17 & 21.91$\pm$1.03 & 25.41$\pm$1.00 & 14.61$\pm$2.03 & 2.10$\pm$0.15 & 45.13 \\
Parkinson & 3.48$\pm$0.44 & 3.69$\pm$0.36 & 3.83$\pm$0.45 & 17.23$\pm$0.55 & 19.24$\pm$0.85 & 7.13$\pm$1.22 &2.23$\pm$0.16 & 45.13\\
Concrete & 4.32$\pm$0.49 & 5.54$\pm$0.71 & 5.43$\pm$0.57 & 45.58$\pm$4.43 & 51.71$\pm$3.69 & 14.87$\pm$4.19 & 4.70$\pm$0.36 & 90.63 \\
\bottomrule
\end{tabular}
}
\end{center}
\end{table}

The results for the decision trees can be found in \cref{tab:decision_trees} and \cref{tab:decision_trees_validation_loss}. Using only P2L to train the trees leads to underfitted trees, as the model is not complex enough to use only a few datapoints to train a complete model.

\begin{table}[ht]
\caption{Results for the decision trees trained using P2L. We report the RMSE achieved by the models with the smallest validation loss. The ratio $\m/n$ is presented in percents (\%).} \label{tab:decision_trees_validation_loss}
\begin{center}
 \resizebox{\columnwidth}{!}{%
\begin{tabular}{ccccccccc}
\toprule
Dataset & Train loss & Validation loss &	Test loss &	$\kl$ bound	&Linear bound	&$\m/n$&	Baseline test loss & $\ell^{\max}$ \\
\midrule
Powerplant & 8.85$\pm$0.98 & 9.02$\pm$0.84 & 9.17$\pm$1.01 & 24.74$\pm$1.42 & 29.19$\pm$1.68 & 1.86$\pm$0.51 & 4.07$\pm$0.13 & 90.6 \\
Infrared & 0.24$\pm$0.02 & 0.27$\pm$0.03 & 0.31$\pm$0.04 & 1.76$\pm$0.05 & 2.08$\pm$0.04 & 8.39$\pm$0.48 & 0.27$\pm$0.03 & 4.26\\
Airfoil & 6.54$\pm$0.89 & 6.55$\pm$0.54 & 6.41$\pm$0.44 & 28.94$\pm$0.71 & 30.73$\pm$0.65 & 15.70$\pm$2.83 & 3.01$\pm$0.19 & 45.13 \\
Parkinson & 12.19$\pm$1.82 & 12.20$\pm$1.73 & 12.09$\pm$2.01 & 21.91$\pm$0.48 & 22.18$\pm$0.55 & 2.47$\pm$1.31 &3.20$\pm$0.15 & 41.37\\
Concrete & 10.24$\pm$1.13 & 10.93$\pm$1.43 & 10.69$\pm$0.98 & 59.40$\pm$2.03 & 62.39$\pm$1.84 & 19.83$\pm$3.38 &6.22$\pm$0.91 & 90.63\\
\bottomrule
\end{tabular}
}
\end{center}
\end{table}

When computing the generalization bounds, we need to assume that the loss is bounded (for the kl bound) or sub-gaussian (for the linear bound).

Note that the RMSE is not bounded. However, the regression trees cannot predict a value greater (respectively lower) than the highest (respectively lowest) target value found in the training dataset. 
To compute the $\kl$ bound, we work under the assumption that the target space $\calY$ is such that
 \begin{equation*}
     \Ycal = \left[  y_- - \tfrac{1}{10}\qty(y_+ - y_-), y_+ - \tfrac{1}{10}\qty(y_+ - y_-) \right]
      \mbox{ with } y_- = \min_{(\bx,y)\in S} y \mbox{ and } y_+ = \max_{(\bx,y)\in S} y\,.
 \end{equation*}
To compute the linear bound, we assume that the distribution is $\varsigma^2$-sub-Gaussian with
\begin{equation*}
   \varsigma = \tfrac{1}{2}\qty(y_+ - y_-).
\end{equation*}
We report the assumed lower and greater values in \cref{tab:regression_dataset_min_max}.

\begin{table}[ht]
\caption{Minimum and maximum target values used to compute the bound on regression problems.} \label{tab:regression_dataset_min_max}
\begin{center}
\begin{tabular}{ccccccc}
\toprule
Dataset & Assumed lower  	& Observed minimum & Observed maximum &	Assumed upper  \\
 & target value	& target value (in $S$) & target value (in $S$) &	 target value \\
\midrule
Powerplant &  412.71 & 420.26 & 495.76 & 503.31 \\
Infrared & 35.40 & 35.75 & 39.3 & 39.66 \\
Airfoil &  99.62 & 103.38 & 140.99 & 144.75\\
Parkinson &  1.59 & 5.04 & 39.51 & 42.96\\
Concrete &  0 & 2.33 & 82.6 & 90.63\\
\bottomrule
\end{tabular}
\end{center}
\end{table}

\subsubsection{Amazon Polarity}
We trained DistilBERT on the Amazon Reviews Polarity dataset. We use a subset of 10\% of the real dataset, amounting to 360k datapoints. 180k are used for pretraining the model, 144k for training and 36k for validation. We use the given test set of 400k datapoints. Using P2L, we add 32 datapoints at a time to the compression set and stop the training when the validation loss hasn't decreased in 20 iterations. The initial model $h_0$ is the model pretrained on the 180k datapoints. For both P2L and the baseline, which was trained for 200 epochs or until it reached 0 errors on the training dataset, we use the following hyperparameter grid : 
\begin{multicols}{2}
\begin{itemize}
    \item Number of pretraining epochs : $\{2,5\}$
    \item Pretraining learning rate : $2\times 10^{-5}$
    \item Dropout probability : $\{0.1, 0.2\}$
    \item Training learning rate : $\{10^{-6}, 10^{-7}, 10^{-8}\}$
\end{itemize}
\end{multicols}

\section{THEORETICAL RESULTS FROM THE LITERATURE}

\begin{corollary}[\citet{shah_margin-sparsity_2005}, Corollary 1]\label{corr:binom_approx}
    For any distribution $\calD$ over $\calX \times \calY$, for any set of messages $\{M(\bfi)\, \forall \bfi \in I\}$, for any deterministic reconstruction function $\scriptR$ that
outputs sample-compressed predictors $h \in \calH$ and for any $\delta \in (0,1]$, with probability at least $1-\delta$ over the draw of $S \sim \calD^n$, we have
    \begin{align*}
        \forall \bfi \in I, \forall \sigma \in M(\bfi)\,: &\ R_{\calD}(\scriptR(S_{\bfi}, \sigma)) \leq 1-\exp\qty(\frac{-1}{n-\m-\kappa} \qty[\ln\mqty(n-\m \\ \kappa) + \ln \mqty(n \\ \m) + \ln \qty(\frac{1}{\zeta(\m)P_{M(\bfi)}(\sigma) \delta})])\,,
    \end{align*}
    with $\kappa = |\mathbf{i}^c|R_{S_{\bfi^c}}(\scriptR(S_{\bfi}, \sigma))$.
\end{corollary}

\begin{theorem}[\citet{perez2021tighter}]\label{thm:pbb}
    For any distribution $\calD$ over $\calX \times \calY$, for any set $\calH$ of predictors $h : \calX \to \calY$, for any loss $\ell:\calH \times \calX \times \calY \to [0,1]$, for any dataset-independent prior distribution $\prior$ on $\calH$, for any $\delta, \delta' \in (0,1]$, with probability at least $1-\delta-\delta'$ over the draw of $S \sim \calD^n$ and a set of $m$ predictors $h_1, \ldots, h_m \sim \posterior_{S}$, where $\posterior_{S}$ is a dataset-dependent posterior distribution over $\calH$, we have 
    \begin{equation*}
        \E_{h \sim \posterior} \calL_{\calD}(h) \leq \kl^{-1}\qty(\kl^{-1}\qty(\frac{1}{m}\sum_{i=1}^m\hatL_{S}(h_i), \frac{1}{m}\log\frac{2}{\delta'} ), \frac{1}{n} \qty[\KL(\posterior ||\prior) + \ln \qty(\frac{2\sqrt{n}}{\delta})]) \,.
    \end{equation*}
\end{theorem}

\section{PROOFS}\label{app:proofs}

\subsection{Proof of the main result}

Before proving \cref{thm:main_results}, we restate Chernoff's bound in a way that will be useful to prove \cref{thm:main_results}. 
\begin{lemma}[Chernoff's bound]\label{eq:chernoff}
    For $t > 0$ and $X$ a random variable :
    \begin{equation*}
        \Prob\qty(X \leq \frac{1}{t}\qty[\ln \E e^{tX} + \ln\frac{1}{\delta}]) \ \geq\ 1-\delta.
    \end{equation*}
\end{lemma}
\begin{proof}[Proof of \cref{eq:chernoff}]
Chernoff's bound states that for a random variable $X$, any $t>0$ and $\epsilon > 0$,  we have
    \begin{align*}
        \Prob\qty(X > \epsilon) \ \leq\ e^{-t \epsilon} \E e^{t X}\,.
    \end{align*}
By choosing $\delta = e^{-t \epsilon} \E e^{t X}$, we have
\begin{align*}
    & \delta = e^{-t \epsilon} \E e^{t X} \\ 
    \iff & e^{t \epsilon} = \frac{1}{\delta}\E e^{t X} \\ 
    \iff & t \epsilon = \ln \frac{1}{\delta}\E e^{t X} \\ 
    \iff & \epsilon = \frac{1}{t} \qty[\ln\E e^{t X} + \ln\frac{1}{\delta}]\,.
\end{align*}
Thus, we obtain
\begin{equation*}
    \Prob\qty(X > \frac{1}{t} \qty[\ln\E e^{t X} + \ln\frac{1}{\delta}]) \leq \delta\,.
\end{equation*}
\end{proof}

\longtrue
\mainresult*

\begin{proof}[Proof of \cref{thm:main_results}]
We start by defining the set of sample-compressed predictors. Given a dataset $S \sim \calD^n$ and~$\calH$ a predictor set, we consider the following subset of $\calH$, that contains only sample-compressed predictors : 
\begin{equation*}
    \widehat{\calH}_{S} \coloneqq \left\{ \scriptR(S_{\bfi}, \sigma) | \bfi \in \scriptP(n), \sigma \in M(\bfi) \right\} \subseteq \calH.
\end{equation*}
Note that for any vector of indices $\bfi \in \scriptP(n)$ and any message $\sigma \in M(\bfi)$, when given 
 a dataset $S$, we obtain a predictor $\scriptR(S_{\bfi}, \sigma) \in \widehat{\calH}_{S}$. 
 
 For a specific pair $(\bfi, \sigma)$, let's study the value of $\Delta\qty(\hatL_{S_{\bfi^c}}(\scriptR(S_{\bfi}, \sigma)), \calL_{\calD}(\scriptR(S_{\bfi}, \sigma)))$, a realization of a random variable of mean
\begin{equation*}
    \E_{T \sim \calD^{n}} \Delta\qty(\hatL_{T_{\bfi^c}}(\scriptR(T_{\bfi}, \sigma)), \calL_{\calD}(\scriptR(T_{\bfi}, \sigma))) = \E_{T_{\bfi} \sim \calD^{\m}} \E_{T_{\bfi^c} \sim \calD^{n-\m}} \Delta\qty(\hatL_{T_{\bfi^c}}(\scriptR(T_{\bfi}, \sigma)), \calL_{\calD}(\scriptR(T_{\bfi}, \sigma))).
\end{equation*} 
With $\delta_{\bfi}^{\sigma} \in (0,1)$ and $t > 0$, using Chernoff's bound as stated in \cref{eq:chernoff}, we have
\begin{align*}
    &\Prob_{S \sim \calD^n} \qty(\Delta\qty(\hatL_{S_{\bfi^c}}(h_{\bfi}^{\sigma}), \calL_{\calD}(h_{\bfi}^{\sigma})) \leq \frac{1}{t} \qty[\ln \E_{T_{\bfi} \sim \calD^{\m}} \E_{T_{\bfi^c} \sim \calD^{n-\m}} e^{t\Delta\qty(\hatL_{T_{\bfi^c}}(\scriptR(T_{\bfi}, \sigma)), \calL_{\calD}(\scriptR(T_{\bfi}, \sigma)))} + \ln \frac{1}{\delta_{\bfi}^{\sigma}}]) \\ 
    &\geq 1-\delta_{\bfi}^{\sigma}\,.
\end{align*}
Thanks to the union bound, we get a bound that is valid for all pairs $(\bfi, \sigma)$ simultaneously,
\begin{align} 
&\Prob_{S \sim \calD^n} \qty(\forall \bfi \in \scriptP(n), \sigma \in M(\bfi) :\Delta\qty(\hatL_{S_{\bfi^c}}(h_{\bfi}^{\sigma}), \calL_{\calD}(h_{\bfi}^{\sigma})) \leq \frac{1}{t} \qty[\ln \E_{T_{\bfi} \sim \calD^{\m}} \E_{T_{\bfi^c} \sim \calD^{n-\m}} e^{t\Delta\qty(\hatL_{T_{\bfi^c}}(\scriptR(T_{\bfi}, \sigma)), \calL_{\calD}(\scriptR(T_{\bfi}, \sigma)))} + \ln \frac{1}{\delta_{\bfi}^{\sigma}}]) \notag \\
&\geq 1-\sum_{\bfi \in \scriptP(n)}\sum_{\sigma \in M(\bfi)}\delta_{\bfi}^{\sigma}.\label{eq:union_bound_chernoff}
\end{align}

    Given $\delta \in (0,1)$,  we set $\delta_{\bfi}^{\sigma} = \mqty(n \\ \m)^{-1} \zeta (\m)P_{M(\bfi)}(\sigma) \delta$ and we obtain
    \begin{align*}
        \sum_{\bfi \in I} \sum_{\sigma \in M(\bfi)} \delta_{\bfi}^{\sigma} &= \sum_{\bfi \in I} \sum_{\sigma \in M(\bfi)} \mqty(n \\ \m)^{-1} \zeta(\m) P_{M(\bfi)}(\sigma) \delta \\ 
        &\leq \sum_{\bfi \in I} \mqty(n \\ \m)^{-1} \zeta(\m) \delta \\
        &= \sum_{m=1}^n \sum_{\bfi \in I_m} \mqty(n \\ \m)^{-1} \zeta(\m) \delta\\
        &= \sum_{m=1}^n \zeta(\m) \delta \\ 
        & \leq \delta.
    \end{align*}
    Thus, we have $1-\sum_{\bfi \in I}\sum_{\sigma \in M(\bfi)}\delta_{\bfi}^{\sigma} \geq 1-\delta$. We substitute $\delta_{\bfi}^{\sigma}$ by $\smqty(n \\ \m)^{-1} \zeta (\m)P_{M(\bfi)}(\sigma) \delta$ and let  $t=n-\m$ in Equation~\eqref{eq:union_bound_chernoff} to finish the proof.
\end{proof}

Note that the proof of \cref{thm:main_results} relies on specific choices for the values of the variables $\delta_{\bfi}^{\sigma}$ and $t$ in Equation~\eqref{eq:union_bound_chernoff}. The next paragraphs discuss the rationales behind these choices.

\paragraph{The choice of $t=n-\m$.} To turn Equation~\eqref{eq:union_bound_chernoff} into a computable bound, one  needs to either compute or upper bound the following term:
    \begin{equation*}
        \E_{T_{\bfi} \sim \calD^{\m}} \E_{T_{\bfi^c} \sim \calD^{n-\m}} e^{t\Delta\qty(\hatL_{T_{\bfi^c}}(\scriptR(T_{\bfi}, \sigma)), \calL_{\calD}(\scriptR(T_{\bfi}, \sigma)))}.
    \end{equation*}
    As the set $T$ is a realization of a $n-\m$ datapoints, choosing $t=n-\m$  generally ensure that this term is bounded. Indeed, this is a requirement for the proofs of multiple results in the PAC-Bayesian theory.

\paragraph{The choice of $\delta_{\bfi}^{\sigma} = {n \choose \m}^{-1} \zeta (\m)P_{M(\bfi)}(\sigma) \delta$.}
Importantly, the value of $\delta_{\bfi}^{\sigma}$ needs to be defined independently of~$S$. Thus, choosing $\delta_{\bfi}^{\sigma}$ is equivalent to choosing the prior distributions $P_{\scriptP(n)}$ and $P_{M(\bfi)}$ in order to obtain 
\begin{equation}\label{eq:delta_i} 
\delta_{\bfi}^{\sigma}= P_{\scriptP(n)}(\bfi)P_{M(\bfi)}(\sigma) \delta.    
\end{equation}
    
    Consider the set $I_m = \{\bfi \in \scriptP(n) : \m = m\}$. As we have no information on which $\bfi \in I_m$ is likely to lead to a good compression set for the reconstruction $\scriptR$, we define a uniform distribution over all $\smqty(n \\ m)$ vectors in $I_m$, which gives a weight of $\smqty(n \\ m)^{-1} \forall \bfi \in I_m$. Now, we want consider all possible sizes of compression set 
    \begin{equation*}
        I = \bigcup_{k=0}^n I_k,
    \end{equation*}
    so we need to define a probability distribution over each set $I_k$. We could simply choose $\tfrac{1}{n+1}$, but the probabilities would tend very fast to zero when we consider a large number $n$ of compression set sizes. 
        It is a better choice, as discussed by \cite{marchand2005learning} in Section 5.2, to choose :
    \begin{equation*}
        \zeta(m) = \frac{6}{\pi^2(m+1)^2}, \quad \text{ for which } \sum_{m=0}^\infty \zeta(m) = 1.
    \end{equation*}
Then, Equation~\eqref{eq:delta_i} is applied with $P_{\scriptP(n)}(\bfi)=\smqty(n \\ \m)^{-1} \zeta(\m)$. To apply the theorem, one must also provide choice of prior distributions $P_{M(\bfi)}$ over the messages $M(\bfi)$  such that $\sum_{\sigma \in M(\bfi)} P_{M(\bfi)}(\sigma) \leq 1$ for all $\ibf\in \scriptP(n)$.

\subsection{Corollaries to the main result}

To prove most corollaries, we are going to need the following lemma.
\begin{lemma}[\citet{maurer2004note}, \citet{risk_bounds}]\label{lemma:real_valued_binomial}
   Let $X$ be any random variable with values in $[0,1]$ and expectation $\mu = \E (X)$. Denote $X$ the vector containing the results of $n$ independent realizations of $X$. Then, consider a Bernoulli random variable $X'$ ($\{0,1\}$-valued) of probability of success $\mu$. Denote $X' \in \{0,1\}^n$ the vector containing the results of $n$ independent realizations of $X'$. 

   If function $g : [0,1]^n \to \R$ is convex, then 
   \begin{equation*}
       \E[g(X)] \ \leq\ \E [g(X')]\,.
   \end{equation*}
\end{lemma}

We now prove our first corollary.
\catoni*
\begin{proof}[Proof of \cref{corr:catoni}]
The proof is divided in two steps: bounding $\Ecal_{\Delta_C}$ and rearranging the terms. Note that these both steps are common in the proofs of PAC-Bayesian literature.  Our proof mainly follow the one of \citet{risk_bounds}.

We start by bounding $\Ecal_{\Delta_C}$.  Let us introduce a random variable $X_{\bfi}^{\sigma}$ that follows a binomial distribution of $m$ trials with a probability of success $\calL_{\calD}(\scriptR(T_{\bfi}, \sigma))$, denoted $B(m, \calL_{\calD}(\scriptR(T_{\bfi}, \sigma)))$. We use \cref{lemma:real_valued_binomial} with $g(\cdot) = e^{m\Delta_C(\cdot, \calL_{\calD}(h))}$.
{\allowdisplaybreaks[4]
\begin{align*}
        \Ecal_{\Delta_C}(\bfi, \sigma) &= \E_{T_{\bfi} \sim \calD^{\m}} \E_{T_{\bfi^c} \sim \calD^{n-\m}} e^{(n-\m)\Delta_C\qty(\hatL_{T_{\bfi^c}}(\scriptR(T_{\bfi}, \sigma)), \calL_{\calD}(\scriptR(T_{\bfi}, \sigma)))} \\ 
        &\leq \E_{T_{\bfi} \sim \calD^{\m}} \E_{X_{\bfi}^{\sigma} \sim B(m, \calL_{\calD}(\scriptR(T_{\bfi}, \sigma)))} e^{(n-\m)\Delta_C\qty(\frac{1}{n-\m}X_{\bfi}^{\sigma}, \calL_{\calD}(\scriptR(T_{\bfi}, \sigma)))} \\ 
        &= \E_{T_{\bfi} \sim \calD^{\m}} \sum_{k=0}^{n-\m} \Prob_{X_{\bfi}^{\sigma} \sim B(m, \calL_{\calD}(\scriptR(T_{\bfi}, \sigma)))}(X_{\bfi}^{\sigma} = k)e^{(n-\m)\Delta_C\qty(\frac{k}{n-\m}, \calL_{\calD}(\scriptR(T_{\bfi}, \sigma)))} \\ 
        &= \E_{T_{\bfi} \sim \calD^{\m}} \sum_{k=0}^{n-\m} \mqty(n-\m \\ k) \qty(\calL_{\calD}(\scriptR(T_{\bfi}, \sigma)))^k \qty(1-\calL_{\calD}(\scriptR(T_{\bfi}, \sigma)))^{n-\m-k} e^{(n-\m)\Delta_C\qty(\frac{k}{n-\m}, \calL_{\calD}(\scriptR(T_{\bfi}, \sigma)))} \\ 
        &\leq \E_{T_{\bfi} \sim \calD^{\m}} \sup_{r \in [0,1]} \qty[\sum_{k=0}^{n-\m} \mqty(n-\m \\ k) \qty(r)^k \qty(1-r)^{n-\m-k} e^{(n-\m)\Delta_C\qty(\frac{k}{n-\m}, r)} ]\\   
        &= \sup_{r \in [0,1]} \qty[\sum_{k=0}^{n-\m} \mqty(n-\m \\ k) \qty(r)^k \qty(1-r)^{n-\m-k} e^{(n-\m)\Delta_C\qty(\frac{k}{n-\m}, r)}] \\   
        &= \sup_{r \in [0,1]} \qty[\sum_{k=0}^{n-\m} \mqty(n-\m \\ k) \qty(r)^k \qty(1-r)^{n-\m-k} \frac{e^{-Ck}}{[1-(1-e^{-C})r]^{n-\m}}] \\   
        &= \sup_{r \in [0,1]} \qty[\sum_{k=0}^{n-\m} \mqty(n-\m \\ k) \qty(re^{-C})^k \qty(1-r)^{n-\m-k} \frac{1}{[1-(1-e^{-C})r]^{n-\m}}] \\   
        &= \sup_{r \in [0,1]} \qty[\frac{[re^{-C} + (1-r)]^{n-\m}}{[1-(1-e^{-C})r]^{n-\m}}] = \sup_{r \in [0,1]} [1] = 1.
\end{align*}
}%
where the last line is derived using binomial theorem.

Now, we rearrange the terms:
{\allowdisplaybreaks[4]
\begin{align*}
   \Delta\qty(\widehat{\mathcal{L}}_{S_{\bfi^c}}(R(\mathbf{i},\sigma)), \mathcal{L}_D(R(\mathbf{i},\sigma))) \leq \frac{1}{n-\m}\qty[\log \mqty(n \\ \m) + \log\qty(\frac{1}{\zeta(\m)P_{M(\bfi)}(\sigma)\delta})] \\ 
   -\ln(1-\mathcal{L}_D(R(\mathbf{i},\sigma))(1-e^{-C})) - C\widehat{\mathcal{L}}_{S_{\bfi^c}}(R(\mathbf{i},\sigma)) \leq \frac{1}{n-\m}\qty[\log \mqty(n \\ \m) + \log\qty(\frac{1}{\zeta(\m)P_{M(\bfi)}(\sigma)\delta})] \\ 
   \ln(1-\mathcal{L}_D(R(\mathbf{i},\sigma))(1-e^{-C})) \geq -C\widehat{\mathcal{L}}_{S_{\bfi^c}}(R(\mathbf{i},\sigma)) - \frac{1}{n-\m}\qty[\log \mqty(n \\ \m) + \log\qty(\frac{1}{\zeta(\m)P_{M(\bfi)}(\sigma)\delta})] \\ 
   1-\mathcal{L}_D(R(\mathbf{i},\sigma))(1-e^{-C}) \geq \exp\qty(-C\widehat{\mathcal{L}}_{S_{\bfi^c}}(R(\mathbf{i},\sigma)) - \frac{1}{n-\m}\qty[\log \mqty(n \\ \m) + \log\qty(\frac{1}{\zeta(\m)P_{M(\bfi)}(\sigma)\delta})]) \\ 
   \mathcal{L}_D(R(\mathbf{i},\sigma))(1-e^{-C}) \leq 1-\exp\qty(-C\widehat{\mathcal{L}}_{S_{\bfi^c}}(R(\mathbf{i},\sigma)) - \frac{1}{n-\m}\qty[\log \mqty(n \\ \m) + \log\qty(\frac{1}{\zeta(\m)P_{M(\bfi)}(\sigma)\delta})]) \\ 
   \mathcal{L}_D(R(\mathbf{i},\sigma)) \leq \frac{1}{1-e^{-C}}\qty[1-\exp\qty(-C\widehat{\mathcal{L}}_{S_{\bfi^c}}(R(\mathbf{i},\sigma)) - \frac{1}{n-\m}\qty[\log \mqty(n \\ \m) + \log\qty(\frac{1}{\zeta(\m)P_{M(\bfi)}(\sigma)\delta})])]
\end{align*}
}
\end{proof}

\klbound*

\begin{proof}
    To prove this corollary, we need to bound $\Ecal_{\kl}$. In the PAC-Bayes literature, this was first done by \citet{langford2001bounds, seeger2002pac} and then improved by \citet{maurer2004note}. We restate the proof of the latter for completeness. 

    We use \cref{lemma:real_valued_binomial} with $g(\cdot) = e^{m\kl(\cdot, \calL_{\calD}(h))}$.
{\allowdisplaybreaks[4]
\begin{align*}
    \Ecal_{\kl}(\bfi, \sigma) &= \E_{T_{\bfi} \sim \calD^{\m}} \E_{T_{\bfi^c} \sim \calD^{n-\m}} e^{(n-\m)\kl\qty(\hatL_{T_{\bfi^c}}(\scriptR(T_{\bfi}, \sigma)), \calL_{\calD}(\scriptR(T_{\bfi}, \sigma)))} \\ 
    &\leq \E_{T_{\bfi} \sim \calD^{\m}} \E_{X_{\bfi}^{\sigma} \sim B(m, \calL_{\calD}(\scriptR(T_{\bfi}, \sigma)))} e^{(n-\m)\kl\qty(\frac{1}{n-\m}X_{\bfi}^{\sigma}, \calL_{\calD}(\scriptR(T_{\bfi}, \sigma)))} \\ 
    &= \E_{T_{\bfi} \sim \calD^{\m}} \sum_{k=0}^{n-\m} \Prob_{X_{\bfi}^{\sigma} \sim B(m, \calL_{\calD}(\scriptR(T_{\bfi}, \sigma)))}(X_{\bfi}^{\sigma} = k)e^{(n-\m)\kl\qty(\frac{k}{n-\m}, \calL_{\calD}(\scriptR(T_{\bfi}, \sigma)))} \\ 
    &= \E_{T_{\bfi} \sim \calD^{\m}} \sum_{k=0}^{n-\m} \mqty(n-\m \\ k) \qty(\calL_{\calD}(\scriptR(T_{\bfi}, \sigma)))^k \qty(1-\calL_{\calD}(\scriptR(T_{\bfi}, \sigma)))^{n-\m-k} e^{(n-\m)\kl\qty(\frac{k}{n-\m}, \calL_{\calD}(\scriptR(T_{\bfi}, \sigma)))} \\ 
    &\leq \E_{T_{\bfi} \sim \calD^{\m}} \sup_{r \in [0,1]} \qty[\sum_{k=0}^{n-\m} \mqty(n-\m \\ k) \qty(r)^k \qty(1-r)^{n-\m-k} e^{(n-\m)\kl\qty(\frac{k}{n-\m}, r)} ]\\   
    &= \sup_{r \in [0,1]} \qty[\sum_{k=0}^{n-\m} \mqty(n-\m \\ k) \qty(r)^k \qty(1-r)^{n-\m-k} e^{(n-\m)\kl\qty(\frac{k}{n-\m}, r)}] \\   
    &= \sup_{r\in[0,1]}\qty[\sum_{k=0}^{n-\m} \mqty(n-\m \\ k) \qty(r)^k \qty(1-r)^{n-\m-k} \times e^{(n-\m) \qty(\frac{k}{n-\m} \ln (\frac{k}{n-\m}\cdot\frac{1}{r}) + (1-\frac{k}{n-\m})\ln ((1-\frac{k}{n-\m})\cdot \frac{1}{1-r})) }] \\
    &= \sup_{r\in[0,1]}\qty[\sum_{k=0}^{n-\m} \mqty(n-\m \\ k) \qty(r)^k \qty(1-r)^{n-\m-k} \times e^{k \ln (\frac{k}{n-\m}\cdot\frac{1}{r}) + (n-\m-k)\ln ((1-\frac{k}{n-\m})\cdot \frac{1}{1-r}) }] \\
    &= \sup_{r\in[0,1]}\qty[\sum_{k=0}^{n-\m} \mqty(n-\m \\ k) \qty(r)^k \qty(1-r)^{n-\m-k} \times e^{
    \ln \qty(\frac{k}{n-\m})^k + \ln \qty(\frac{1}{r})^k 
    + \ln \qty(1-\frac{k}{n-\m})^{n-\m-k} + \ln \qty(\frac{1}{1-r})^{n-\m-k}}]\\
    &= \sup_{r\in[0,1]}\qty[\sum_{k=0}^{n-\m} \mqty(n-\m \\ k) \qty(r)^k \qty(1-r)^{n-\m-k} \times \frac{1}{(r)^k (1-r)^{n-\m-k}} \qty(\frac{k}{n-\m})^k \qty(1-\frac{k}{n-\m})^{n-\m-k}]\\
    &=\sup_{r\in[0,1]}\qty[\sum_{k=0}^{n-\m} \mqty(n-\m \\ k) \qty(\frac{k}{n-\m})^k \qty(1-\frac{k}{n-\m})^{n-\m-k}]\\
    &\leq e^{\frac{1}{12(n-\m)}}\sqrt{\frac{\pi (n-\m)}{2}} + 2 \\
    &\leq 2\sqrt{n-\m}.  
\end{align*}
}%
The last two inequalities were proven by \citet{maurer2004note} for $n-\m \geq 8$. As noticed afterward by \cite{risk_bounds}, it can be verified numerically that $\Ecal_{\kl}(\bfi, \sigma) \leq 2\sqrt{n-\m}$ also holds for $1 \leq n-\m < 8.$.
\end{proof}

\linearloss*
\begin{proof}
We assume that the loss $\ell$ is $\varsigma^2$-sub-Gaussian, which is defined as
\begin{equation*}
\E_{(\bx,y) \sim \calD} \exp\qty[\lambda \qty(\ell(h,\bx,y) - \E_{(\bx', y') \sim \calD} \ell(h,\bx',y'))] \leq \exp\qty(\frac{\lambda^2\varsigma^2}{2}).
\end{equation*}

Then, we have
{\allowdisplaybreaks[4]
\begin{align*}
\Ecal_{\Delta_{\lambda}}(\bfi, \sigma) &= \E_{T_{\bfi} \sim \calD^{\m}} \E_{T_{\bfi^c} \sim \calD^{n-\m}} \exp\qty[(n-\m)\Delta_{\lambda}\qty(\hatL_{T_{\bfi^c}}(\scriptR(T_{\bfi}, \sigma)), \calL_{\calD}(\scriptR(T_{\bfi}, \sigma)))] \\ 
&= \E_{T_{\bfi} \sim \calD^{\m}} \E_{T_{\bfi^c} \sim \calD^{n-\m}} \exp\qty[(n-\m)\lambda\qty(\calL_{\calD}(\scriptR(T_{\bfi}, \sigma))-\hatL_{T_{\bfi^c}}(\scriptR(T_{\bfi}, \sigma)))] \\ 
&= \E_{T_{\bfi} \sim \calD^{\m}} \E_{T_{\bfi^c} \sim \calD^{n-\m}} \exp\qty[(n-\m)\lambda\qty(\E_{(\bx, y) \sim \calD}\ell(\scriptR(T_{\bfi}, \sigma), \bx, y)-\frac{1}{n-\m} \sum_{i=1}^{n-\m} \ell(\scriptR(T_{\bfi}, \sigma), \bx_i, y_i))] \\ 
&= \E_{T_{\bfi} \sim \calD^{\m}} \E_{T_{\bfi^c} \sim \calD^{n-\m}} \exp\qty[\lambda\sum_{i=1}^{n-\m}\qty(\E_{(\bx, y) \sim \calD}\ell(\scriptR(T_{\bfi}, \sigma), \bx, y)-\ell(\scriptR(T_{\bfi}, \sigma), \bx_i, y_i))] \\ 
&= \E_{T_{\bfi} \sim \calD^{\m}} \E_{T_{\bfi^c} \sim \calD^{n-\m}} \exp\qty[-\lambda\sum_{i=1}^{n-\m}\qty(\ell(\scriptR(T_{\bfi}, \sigma), \bx_i, y_i) - \E_{(\bx, y) \sim \calD}\ell(\scriptR(T_{\bfi}, \sigma), \bx, y))] \\ 
&= \E_{T_{\bfi} \sim \calD^{\m}} \E_{T_{\bfi^c} \sim \calD^{n-\m}} \prod_{i=1}^{n-\m} \exp\qty[-\lambda\qty(\ell(\scriptR(T_{\bfi}, \sigma), \bx_i, y_i) - \E_{(\bx, y) \sim \calD}\ell(\scriptR(T_{\bfi}, \sigma), \bx, y))] \\ 
&= \E_{T_{\bfi} \sim \calD^{\m}} \prod_{i=1}^{n-\m} \E_{(\bx_i, y_i) \sim \calD}  \exp\qty[-\lambda\qty(\ell(\scriptR(T_{\bfi}, \sigma), \bx_i, y_i) - \E_{(\bx, y) \sim \calD}\ell(\scriptR(T_{\bfi}, \sigma), \bx, y))] \numberthis \label{eq:iid_assumption}\\ 
&\leq \E_{T_{\bfi} \sim \calD^{\m}} \prod_{i=1}^{n-\m} \exp\qty(\frac{\lambda^2 \varsigma^2}{2}) \numberthis \label{eq:subgaussian} \\ 
&= \exp\qty(\frac{(n-\m)\lambda^2 \varsigma^2}{2}).
\end{align*}
}%

Equation~\eqref{eq:iid_assumption} relies on the \emph{i.i.d.} assumption and Equation~\eqref{eq:subgaussian}, on the $\varsigma^2$-sub-Gaussian assumption.

We replace the comparator function in Theorem~\ref{thm:main_results} and bound the cumulant generating function $\Ecal_{\Delta_{\lambda}}$ to finish the proof.
\end{proof}

\subsection{Behavior with zero error}\label{app:consistent_case}

Let us recall the definition of the binomial tail,  
\begin{equation*}
    \textrm{Bin}(k, m, r) = \sum_{i=0}^{k} \mqty(m \\ i) r^i (1-r)^{m-i},
\end{equation*}
and the binomial tail inversion,
\begin{equation*}
    \overline{\textrm{Bin}}(k,m,\delta) = \argsup_{r \in [0,1]} \left\{ \textrm{Bin}(k, m, r) \geq \delta \right\}.
\end{equation*}

The proofs of \cref{thm:consistent_case_catoni} and \cref{thm:consistent_case_kl} make use of the following lemma.

\begin{lemma}\label{lemma:binom_analytical}
    In the consistent case, the binomial tail has the following analytical solution:
    \begin{equation*}
        \overline{\mathrm{Bin}}(0,m,\delta) = 1-\exp\qty(-\frac{1}{m}\ln\qty(\frac{1}{\delta})).
    \end{equation*}
\end{lemma}

\begin{proof}
We start by rewriting the binomial tail distribution, with $k = 0$.
\begin{align*}
    \textrm{Bin}(0, m, r) &= \sum_{i=0}^{0} \mqty(m \\ i) r^i (1-r)^{m-i} \\
    &= \mqty(m \\ 0) r^0 (1-r)^{m-0} \\
    &= 1 \cdot 1 \cdot (1-r)^m \\
    &= (1-r)^m.
\end{align*}
We now rewrite the binomial tail inversion.
\begin{align*}
\overline{\textrm{Bin}}(0,m,\delta) &= \argsup_{r \in [0,1]} \left\{ \textrm{Bin}(0, m, r) \geq \delta \right\} \\ 
&= \argsup_{r \in [0,1]} \left\{ (1-r)^m \geq \delta \right\} \\ 
&= \argsup_{r \in [0,1]} \left\{ 1-r \geq \delta^{\frac{1}{m}} \right\} \\ 
&= \argsup_{r \in [0,1]} \left\{ 1-\delta^{\frac{1}{m}} \geq r \right\} \\ 
&= 1-\delta^{\frac{1}{m}} \numberthis \label{eq:argsup_eq}\\ 
&= 1-\exp\qty(\ln\qty(\delta^{\frac{1}{m}})) \\ 
&= 1-\exp\qty(-\frac{1}{m}\ln\qty(\frac{1}{\delta})) 
\end{align*}
At \cref{eq:argsup_eq}, we use the fact that the maximum value of $r$ such that $1-\delta^{\frac{1}{m}}\geq r$ is simply $1-\delta^{\frac{1}{m}}= r$.
\end{proof}
We did not find any mention of \cref{lemma:binom_analytical}'s equality in the literature. However, the following inequality is well known and can be found in papers such as \citet{shah_margin-sparsity_2005}:
\begin{equation*}
    \overline{\textrm{Bin}}(k,m,\delta) \leq 1-\exp\qty(\frac{-1}{m-k} \qty[\ln\mqty(m \\ k) + \ln \qty(\frac{1}{\delta})]).
\end{equation*}
When $k=0$, it reduces to 
\begin{equation*}
    \overline{\textrm{Bin}}(0,m,\delta) \leq 1-\exp\qty(\frac{-1}{m}  \ln \qty(\frac{1}{\delta})).
\end{equation*}

We now prove \cref{thm:consistent_case_catoni}.
\consistentcasecatoni*
\begin{proof}
    From \cref{lemma:binom_analytical}, we have 
    \begin{equation*}
        \overline{\textrm{Bin}}(0,m,\delta) = 1-\exp\qty(\frac{-1}{m}  \ln \qty(\frac{1}{\delta})).
    \end{equation*}
    Thus, we will prove the two following equations : 
    \begin{equation}\label{eq:first_thm8}
        1-\exp\qty(-\frac{1}{m} \ln \frac{1}{\delta}) = \lim_{C \to \infty}\left\{ \frac{1}{1-e^{-C}} \qty[1-\exp\qty(-\frac{1}{m} \ln \frac{1}{\delta})]\right\}
    \end{equation}
    and 
    \begin{equation}\label{eq:second_thm8}
    \lim_{C \to \infty}\left\{ \frac{1}{1-e^{-C}} \qty[1-\exp\qty(-\frac{1}{m} \ln \frac{1}{\delta})]\right\}  = \inf_{C > 0} \left\{\frac{1}{1-e^{-C}} \qty[1-\exp\qty(-\frac{1}{m} \ln \frac{1}{\delta})]\right\}
    \end{equation}

To  prove \cref{eq:first_thm8}, we simply compute
    \begin{align*}
        \lim_{C \to \infty} \frac{1}{1-e^{-C}} \qty[1-\exp\qty(-\frac{1}{m} \ln \frac{1}{\delta})] &= \qty[1-\exp\qty(-\frac{1}{m} \ln \frac{1}{\delta})] \lim_{C \to \infty} \frac{1}{1-e^{-C}} \\ 
        &= \qty[1-\exp\qty(-\frac{1}{m} \ln \frac{1}{\delta})] \cdot 1 \\
        &= 1-\exp\qty(-\frac{1}{m} \ln \frac{1}{\delta})
    \end{align*}

Let's now prove \cref{eq:second_thm8}. We know that 
\begin{equation*}
    \inf_{C > 0} \left\{\frac{1}{1-e^{-C}} \qty[1-\exp\qty(-\frac{1}{m} \ln \frac{1}{\delta})]\right\} = \qty[1-\exp\qty(-\frac{1}{m} \ln \frac{1}{\delta})] \inf_{C > 0} \left\{\frac{1}{1-e^{-C}} \right\}
\end{equation*}
We obtain that function $f(C) = \frac{1}{1-e^{-C}}$ is decreasing on $[0,\infty)$, and thus has a minimum at $C \to \infty$, by the fact that its derivative is negative for all $C\geq 0$: 
\begin{align*}
    f'(C) &= \dv{C}\frac{1}{1-e^{-C}} = \dv{C}\qty(1-e^{-C})^{-1} \\ 
    &= \frac{-1}{\qty(1-e^{-C})^{2}} \dv{C}\qty(1-e^{-C}) \\ 
    &= \frac{-1}{\qty(1-e^{-C})^{2}} e^{-C} \\
    &= (-1)  \frac{e^{-C}}{\qty(1-e^{-C})^{2}} .
\end{align*}
\end{proof}
Note that if one tries to find a value of $C \in \R_{>0}$ such that the derivative $f'(C)$ of the above proof is zero, they obtain
\begin{align*}
    f'(C) = (-1)  \frac{e^{-C}}{\qty(1-e^{-C})^{2}} &= 0 \\ 
    \iff  e^{-C} &= 0.
\end{align*}
There is no $C \in \R_{>0}$ such that $e^{-C} = 0$, however we know that $\lim_{C \to \infty} e^{-C} = 0$. 

Thus, as the function is monotonically decreasing when $C$ increases, given an arbitrarily small $\epsilon >0$, there is always a $C(\epsilon)$ large enough such that 
\begin{equation*}
    \frac{1}{1-e^{-C(\epsilon)}}-1 \leq \epsilon.
\end{equation*}
For example, with $\epsilon = 0.01$ and $C(\epsilon)=4.616$, we have $\frac{1}{1-e^{-C(\epsilon)}}-1=0.00999 \leq 0.01$. With $\epsilon=10^{-5}$ and $C(\epsilon) = 11.513$, we have $\frac{1}{1-e^{-C(\epsilon)}}-1=0.999\times 10^{-5} \leq 10^{-5}$. Thus, it is possible to be arbitrarily tight to $1$, for any $\epsilon > 0$.

\bigskip

We now prove \cref{thm:consistent_case_kl}.
\consistentcasekl*

\begin{proof}
To prove this result, we need to prove the following sequence of equations:
\begin{align}
    \overline{\textrm{Bin}}(0,m,\delta) &= \kl^{-1}\qty(0, \frac{1}{m}\ln \frac{1}{\delta}) \label{eq:first_thm9}\\
    &\leq  \kl^{-1}\qty(0, \frac{1}{m} \ln \frac{2\sqrt{m}}{\delta}) \label{eq:second_thm9}\\
    &=\kl^{-1}\qty(0, \frac{1}{m} \ln \frac{1}{\delta}) + K(m,\delta).\label{eq:third_thm9}
\end{align}

We first prove \cref{eq:first_thm9}. We already know that
\begin{equation*}
    \overline{\textrm{Bin}}(0,m,\delta) = \inf_{C > 0} \left\{\frac{1}{1-e^{-C}} \qty[1-\exp\qty(-\frac{1}{m} \ln \frac{1}{\delta})]\right\}.
\end{equation*}

Moreover, from the proof of Theorem~3 of \citet{letarte2019dichotomize}, for any constant $A > 0$, we know that 
\begin{equation}\label{eq:gael}
    \inf_{C > 0} \left\{\frac{1}{1-e^{-C}} \qty[1-\exp\qty(-C \qty[\frac{k}{m}]-\frac{1}{m} \ln \frac{A}{\delta})]\right\} = \kl^{-1}\qty(\frac{k}{m}, \frac{1}{m} \ln \frac{A}{\delta}).
\end{equation}

Thus, with $A=1$, we have 
\begin{equation*}
    \overline{\textrm{Bin}}(0,m,\delta) = \kl^{-1}\qty(0, \frac{1}{m} \ln \frac{1}{\delta}).
\end{equation*}

We now prove \cref{eq:second_thm9}.

The function $\kl(0,p)$ is monotonically increasing and $\kl(0,1) = \infty$. Thus, there exists a value $p^* = \kl^{-1}\qty(0, \frac{1}{m} \ln \frac{1}{\delta})$ such that $\kl(0,p^*) = \frac{1}{m} \ln \frac{1}{\delta}$. Moreover, there exists a value $p^\dagger = \kl^{-1}\qty(0, \frac{1}{m} \ln \frac{2\sqrt{m}}{\delta})$ such that $\kl(0,p^{\dagger}) = \frac{1}{m} \ln \frac{2\sqrt{m}}{\delta}$. As $\kl(0,p)$ is monotonically increasing and 
\begin{equation*}
    \kl(0,p^*) = \frac{1}{m} \ln \frac{1}{\delta} \leq \frac{1}{m} \ln \frac{2\sqrt{m}}{\delta} = \kl(0,p^{\dagger}),
\end{equation*}
then $p^* \leq p^{\dagger}$.

Finally, we prove \cref{eq:third_thm9} and show that $K(m,\delta)$ tends to $0$ when $m$ tends to $\infty$ and is bounded by
\begin{equation*}
    0\leq K(m,\delta) \leq \frac{1}{m}\ln\frac{2\sqrt{m}}{\delta}.
\end{equation*}

We start by defining the constant $K(m,\delta)$ as the gap between the two following terms: 
\begin{align*}
    &\kl^{-1}\qty(0, \frac{1}{m} \ln \frac{2\sqrt{m}}{\delta}) - \kl^{-1}\qty(0, \frac{1}{m} \ln \frac{1}{\delta}) \\ 
    =&  \qty[1-\exp\qty(-\frac{1}{m} \ln \frac{2\sqrt{m}}{\delta})] - \qty[1-\exp\qty(-\frac{1}{m} \ln \frac{1}{\delta})] \\ 
    =&\exp\qty(-\frac{1}{m} \ln \frac{1}{\delta}) - \exp\qty(-\frac{1}{m} \ln \frac{2\sqrt{m}}{\delta})\\ 
    \eqqcolon& K(m,\delta)
\end{align*}
In the second line, we use \cref{eq:gael} with both $A=2\sqrt{m}$ and $A=1$.
\end{proof}

We now highlight some properties of $K(m,\delta)$. First of all, we show that $K(m,\delta) \geq 0$ for $m \geq \tfrac{1}{4}$.
\begin{align*}
    & \exp\qty(-\frac{1}{m}\ln\frac{1}{\delta}) - \exp\qty(-\frac{1}{m}\ln\frac{2\sqrt{m}}{\delta}) \geq 0 \\ 
    \iff & \exp\qty(-\frac{1}{m}\ln\frac{1}{\delta}) \geq \exp\qty(-\frac{1}{m}\ln\frac{2\sqrt{m}}{\delta})  \\ 
    \iff & -\frac{1}{m}\ln\frac{1}{\delta} \geq -\frac{1}{m}\ln\frac{2\sqrt{m}}{\delta} \numberthis \label{eq:increasing_1} \\
    \iff & \frac{1}{m}\ln\frac{2\sqrt{m}}{\delta} \geq \frac{1}{m}\ln\frac{1}{\delta} \\
    \iff & 2\sqrt{m} \geq 1 \numberthis \label{eq:increasing_2} \\ 
    \iff & m \geq \frac{1}{4} 
\end{align*}
In \cref{eq:increasing_1} and \cref{eq:increasing_2}, we use the fact that both the exponentials and logarithms are increasing functions. From this, we also know that $K(\tfrac{1}{4}, \delta) = 0.$ As the parameter $m$ is the size of a dataset, we know that we always have $m \geq 1$ and thus $K(m,\delta) \geq 0$.

Secondly, we show that $K(m,\delta)$ tends to $0$ when $m$ tends to $\infty$. 
\begin{align*}
    \lim_{m \to \infty} K(m,\delta) &= \lim_{m \to \infty}\left\{\exp\qty(-\frac{1}{m} \ln \frac{1}{\delta}) - \exp\qty(-\frac{1}{m} \ln \frac{2\sqrt{m}}{\delta})\right\} \\ 
    &= \lim_{m \to \infty}\left\{\exp\qty(-\frac{1}{m} \ln \frac{1}{\delta})\right\} - \lim_{m \to \infty} \left\{\exp\qty(-\frac{1}{m} \ln \frac{2\sqrt{m}}{\delta})\right\} \\ 
    &= 1 - 1 = 0.
\end{align*}

Next, we compute a simple upper bound of $K(m,\delta)$, that also tends to $0$ when $m \to \infty$.
\begin{align*}
    K(m,\delta) &=\exp\qty(-\frac{1}{m} \ln \frac{1}{\delta}) - \exp\qty(-\frac{1}{m} \ln \frac{2\sqrt{m}}{\delta}) \\
    &\leq 1 - \exp\qty(-\frac{1}{m} \ln \frac{2\sqrt{m}}{\delta}) \\
    &\leq \frac{1}{m} \ln \frac{2\sqrt{m}}{\delta}.
\end{align*}
The last line uses the inequality $1-e^{-x} \leq x$.

\end{document}